\documentclass[sigconf]{acmart}
\usepackage{color}
\usepackage{float}
\usepackage{framed}
\usepackage{amsmath}
\usepackage{amsthm} 

\usepackage{amssymb}
\usepackage{bbm}
\usepackage{graphicx}
\usepackage[ruled,vlined]{algorithm2e}
\usepackage{algorithmic}
\usepackage{makecell} 
\usepackage{booktabs}
\usepackage{comment}
\usepackage{threeparttable}
\usepackage{tabularx} 
\usepackage{pifont}
\newcommand{\tick}{\ding{51}} 
\newcommand{\cross}{\ding{55}} 

\usepackage{graphicx,subcaption}
\usepackage[skip=0.333\baselineskip]{subcaption}
\allowdisplaybreaks

\newtheorem{theorem}{Theorem}
\newtheorem{lemma}[theorem]{Lemma}

\newtheorem{proposition}[theorem]{Proposition}

\theoremstyle{remark}
\newtheorem{remark}{Remark}
\ifodd 1

\else

    \newcommand{\com}[1]{}
\fi

\newcommand\Qiulin[1]{{#1}}
\newcommand\com[1]{{\color{red}#1}}

\def \ISTR {}

\copyrightyear{2026}\acmYear{2026}
\acmConference[MobiHoc '26]{The Twenty-seventh International Symposium on Theory, Algorithmic Foundations, and Protocol Design for Mobile Networks and Mobile Computing}{November 23-26, 2026}{Tokyo, Japan}
\acmBooktitle{The Twenty-seventh International Symposium on Theory, Algorithmic Foundations, and Protocol Design for Mobile Networks and Mobile Computing (MobiHoc '26), November 23-26, 2026, Tokyo, Japan}
\acmDOI{10.1145/3842721.3849669}
\acmISBN{979-8-4007-2975-1/2026/11}
\acmPrice{}

\title{Learning When to Update: A Near-Optimal Timing Bandit Approach}

\author{Qiulin Lin, Junyan Su, Liyuan Wang}
\affiliation{City University of Hong Kong,\country{Hong Kong, China}}

\author{Minghua Chen}
\affiliation{The Chinese University of Hong Kong (Shenzhen)\country{China} \\ City University of Hong Kong\country{Hong Kong, China}}

\begin{document}

\begin{abstract}


Systems operating in dynamic environments require timely updates to sustain performance. For resource-intensive systems such as machine learning models and digital twins, strategically timing updates is essential. Updating too frequently wastes resources, while updating too infrequently leads to costly performance degradation. The problem is particularly challenging when the system's degradation pattern is unknown a priori, as is common in new operating environments. 
We formalize this challenge as a novel \emph{timing bandit} problem, where each arm represents a candidate update interval with a fixed update cost and an unknown, stochastic degradation cost. Three structural properties distinguish this setting from standard multi-armed bandits: selecting an interval commits the learner to multiple time slots before the next update; arm costs are composed of per-step degradation costs and a fixed update cost; and selecting a longer interval naturally reveals degradation at every intermediate step, providing consecutive feedback relevant to shorter intervals. By exploiting these structures, we develop Balanced Consecutive Arm Elimination (\textsf{BCAE}). \textsf{BCAE} achieves $\tilde{O}(\sqrt{T})$ regret, improving upon the $\tilde{\Omega}(K\sqrt{T})$ regret of standard bandit algorithms in this setting, where $K$ is the number of candidate update intervals. We further propose an Optimism-Enhanced variant (\textsf{OE-BCAE}) that integrates lower-confidence-bound principles to improve empirical adaptivity while preserving the same regret order. Moreover, the regret bound achieved by our algorithms matches the theoretical lower bound up to logarithmic factors. {Simulation results demonstrate that our algorithms achieve low regret and remain stable as both the number of arms and the update cost vary.}

\end{abstract}


\maketitle

\section{Introduction}
\label{sec:intro}

Timely updates are crucial for systems operating in dynamic environments to sustain performance. For instance, machine learning models need to be retrained or fine-tuned to counteract model drift~\cite{lu2018learning}, a phenomenon in which the statistical properties of input data change over time since original training, leading to a gradual degradation in predictive performance~\cite{vela2022temporal,Lee2023Recommender}. Digital twins of physical assets need frequent synchronization with the real-world state they mirror. Otherwise, the fidelity of simulation-based decisions
degrades~\cite{grieves2016digital,tao2018digital}. 

However, performing an update is seldom free. Retraining a large-scale machine learning model may consume substantial computational resources
and energy. Synchronizing a digital twin can require expensive sensor readings or communication bandwidth. More specifically, in networked and edge deployments, these costs are often dominated by communication overhead. Triggering an update may require waking up a radio transceiver or transmitting a large volume of data over a
bandwidth-limited link, each of which drains the limited battery budget of an IoT device or edge node~\cite{wang2019adaptive, Bacinoglu2019finiteJournal}. Consequently, a system designer faces a fundamental trade-off: updating too frequently wastes resources on unnecessary refreshes, while updating too infrequently allows performance to degrade, potentially incurring costs that far exceed the savings from skipping an update.

Striking the right balance is further complicated by the fact that degradation patterns are often \emph{unknown a priori}. The rate at which a model's accuracy declines, or a digital twin's fidelity erodes, depends on the volatility of the underlying environment---a quantity that is itself uncertain and may change across deployments. When the degradation cost associated with each possible update interval is known or can be modeled parametrically, the optimal update schedule can be computed offline. In practice, however, such prior information is rarely available, especially when a system is deployed in a new environment. This renders the problem inherently an online learning problem: the system must discover an effective update policy through its own operational experience. The Age of Information (AoI) literature has extensively studied update scheduling under various application scenarios~\cite{sun2017update,yates2021age}, but it predominantly assumes known degradation functions (or age penalty functions) and optimizes a
closed-form objective. In contrast, the setting in which the degradation dynamics are stochastic and must be \emph{learned} online remains largely open; existing studies either assume full knowledge of the degradation model~\cite{Bacinoglu2019finiteJournal, arafa2019age, sun2017update} or adopt fully adversarial formulations~\cite{tripathi2019whittle,tripathi2021onlinearXiv,Lin2025AoI} that do not exploit the stochastic structure; see a detailed discussion of related work in Section~\ref{sec:related-work}. 

In this paper, we investigate the update scheduling problem (\textsf{USP}), which involves strategically determining the timing of updates under realistic conditions where future performance degradation patterns are stochastic and unknown. We aim to minimize the average total expected degradation cost and update expenses over a time horizon. The fundamental challenge lies in learning the unknown degradation costs while making real-time update decisions, without prior knowledge of the degradation cost patterns or their distributions. We frame this challenge as a novel variant of the multi-armed bandit (MAB) problem, named \textit{Timing Bandit}, where each arm corresponds to a candidate update interval. We reveal three structural properties of Timing Bandit. First, \emph{timing decisions}: selecting arm~$k$ commits the learner to operate for the next $k$~time steps before performing an update, so a
single decision occupies a variable number of rounds rather than a single slot. Second, \emph{arm cost composition}: the cost of an arm comprises the stochastic degradation cost that accumulates as the system ages during
the interval and a fixed update cost incurred at the end.
Third, \emph{consecutive feedback}: pulling an arm traverses every intermediate age in sequence, revealing degradation information relevant to every shorter interval. These structures distinguish our problem from existing variants of MAB problems (cf. Table~\ref{tab:bandit-comparison}) and necessitate novel algorithm design and analysis. To this end, we develop online algorithms that achieve provably near-optimal, sublinear regret bounds, where regret is defined as the performance gap relative to an optimal strategy with known degradation cost distributions. We summarize our \textbf{contributions} as follows.

$\rhd$ In Sec.~\ref{sec:model}, we introduce the update scheduling problem (\textbf{USP}), whose objective is to minimize the total cost comprising both stochastic degradation and update expenses by optimizing the timing of updates. We show that the optimal policy for minimizing the asymptotic time-average total cost reduces to a periodic one, i.e., updating at a fixed interval.

$\rhd$ Motivated by the optimality of periodic updates, in Sec.~\ref{sec:mab}, we frame the online setting, where the stochastic degradation costs are unknown a priori, as a novel \textit{Timing Bandit} problem. Here, each arm corresponds to a candidate update interval whose fixed update cost is known but whose degradation costs must be learned through interaction. We further identify three structural properties of \textit{Timing Bandit}, including timing decisions, arm cost composition, and consecutive feedback.  

$\rhd$ In Sec.~\ref {sec:algorithm}, we demonstrate that, resulting from the timing decision structure, a direct application of standard MAB algorithms, i.e., treating each of the $K$~candidate intervals as an independent arm, yields a regret of
$\tilde{O}(K\sqrt{T})$ rather than the usual
$\tilde{O}(\sqrt{KT})$.  We then propose an efficient online algorithm,  namely Balanced Consecutive Arm Elimination (\textsf{BCAE}), for the timing bandit. Built on arm elimination, our algorithm further leverages a novel idea of balancing the confidence gaps across all arms, enabling joint use of the specific arm cost composition and the consecutive feedback structures. We prove that our algorithm achieves sublinear regret bounds of~$\tilde{O}(\sqrt{T})$ for sufficiently large $T$, removing the $K$ multiplicative overhead found in applying standard MAB algorithms.  We show that the achieved regret matches the regret lower bound for the problem up to logarithmic factors.  While analytically convenient, the elimination paradigm adapts only at infrequent, discrete elimination events, limiting empirical responsiveness. To address this, we further propose an Optimism-Enhanced variant (\textsf{OE-BCAE}) that integrates Lower Confidence Bound (LCB) principles into the elimination framework, enabling continuous
adaptation while preserving the same order regret.

$\rhd$ \ifx \ISTR \undefined
We validate our theoretical findings through numerical evaluation in Sec.~\ref{sec:simulation} and evaluate the empirical performance on a remote monitoring instance using real-world traces. 
\else
We validate our theoretical findings through numerical evaluation in Sec.~\ref{sec:simulation} and evaluate the empirical performance on a remote monitoring instance using real-world traces in Sec.~\ref{sec:empirical}. 
\fi
The experiments demonstrate that our approaches, in particular \textsf{OE-BCAE}, achieve lower regret and remain robust to changes in both the number of arms and the update cost.  \ifx \ISTR \undefined
Due to space limitations, the results of the empirical evaluation are provided in our technical report~\cite{TechnicalReport}. 
\else
\fi

\section{Related Work}\label{sec:related-work}

\newcommand{\multicell}[1]{\begin{tabular}[c]{@{}c@{}} #1 \end{tabular}}

\begin{table}[!t]
\caption{Comparison with related stochastic multi-armed bandit (MAB) problems. }
\label{tab:bandit-comparison}
\centering
\begin{threeparttable}
\resizebox{\columnwidth}{!}{%
\begin{tabular}{@{}l@{\hskip 1pt}c@{\hskip 5pt}c@{\hskip 5pt}c@{\hskip 5pt}l@{}}
\toprule
\multicell{\multicell{\textbf{MAB} \\\textbf{Problems}}} & \multicell{\textbf{Timing} \\ \textbf{Decisions}} & \multicell{\textbf{Consecutive} \\ \textbf{Feedback}} & \multicell{  \textbf{Arm Cost} \\\textbf{Composition}} & \multicell{\textbf{Regret}}\\ \midrule
Standard, e.g.,~\cite{lattimore2020bandit}   & \cross               & \cross                         & \cross & $\tilde{O}(\sqrt{KT})^\star$    \\
One-Sided, e.g., \cite{Swapna2018Side,zhao2019,Li2021Waiting} & \cross & \tick & \cross & $\tilde{O}(\sqrt{T})^\star$\\
Side Info.~\cite{caron2012leveraging} & \cross & \tick & \cross & $\tilde{O}(\sqrt{KT})^\dag$\\
\multicell{{Batch}, e.g.,~\cite{Vianney2016Batched,Gao2019Batched,cao2023best} }     & \tick   & \tick             & \cross & {$\tilde{O}(h(\cdot))^\ddag{^\star}$ }\\
 \midrule
 \multicell{\multicell{{\textsf{Timing Bandit}} \\ {(this work)}}}     & \tick              & \tick              & \tick & $\tilde{O}(\sqrt{T})^\star$ 
\\ \bottomrule

\end{tabular}
}
\begin{tablenotes}[flushleft]
\footnotesize
\item $^\star$ Denotes the regret bound matches its lower bound up to logarithmic factors. 
\item $^\dag$ Result when applied to the consecutive or one-sided feedback case.
\item ${^\ddag}$ $h(\cdot)=\sqrt{K}T^{{1/(2-2^{1-N})}}$, where $N$ represents the number of batches. 
\end{tablenotes}

\end{threeparttable}
\vspace{-4mm}
\end{table}

\textbf{Multi-armed Bandit.}
We consider \emph{Timing bandit}, a novel variant of the classical multi-armed bandit (MAB) framework~\cite{slivkins2019introduction} with structural properties tailored to our setting. This variant differs from traditional MAB problems in decision-making, feedback mechanisms, and arm cost composition, which require new algorithm design and analysis, as we will elaborate in Sec.~\ref{sec:mab} and Sec.~\ref{sec:algorithm}. We summarize the comparison of the most-related variants of bandit problems in Table~\ref{tab:bandit-comparison}. While some properties of our problem can be considered individually as special cases of other MAB paradigms, there is no direct generalization from these paradigms that accounts for all of our problem's properties. In particular, in~\cite{caron2012leveraging,Swapna2018Side,zhao2019,Li2021Waiting}, the authors consider bandits in which pulling an arm also provides samples of arms with smaller (or larger) indexes, similar to our consecutive feedback structure. Also, the timing decision structure that pulling an arm, say $k$, occupies $k$ slots and reveals feedback of the first $k$ arms, can be viewed as pulling a batch of the first $k$ arms in the batched multi-armed bandit problem~\cite{Vianney2016Batched,Gao2019Batched,cao2023best}. However, we note that, rather than the learner choosing a subset of arms to observe, it is the temporal physics of the timing bandit that determines the timing of observations, and we cannot choose arbitrary batches. Other examples include that the way the arm cost is composed in our problem can be connected to the combinatorial multi-armed bandit problem~\cite{Chen2013Combinatorial,Chen2016Combinatorial, Chen2016Combinatorial2}, The paper in~\cite{Cayci2019Learning}, which focuses on the control of a renewal reward process, exhibits the property of consecutive feedback, where choosing a longer interrupt time also reveals feedback of a shorter one. {Here, we exploit all three properties of the problem with novel algorithmic design and achieve a provable, improved sublinear regret, which matches the lower bound regret within logarithmic factors.} 

\textbf{Age-of-Information Optimization.}
Our problem is also related to the Age-of-Information (AoI) optimization in communication networks~\cite{kaul2012real,yates2021age}. AoI captures the freshness of information received at the destination, and the general goal for AoI optimization is to maintain a low AoI by optimizing the update process. Our approach extends to AoI optimization scenarios where the age-dependent penalty follows some unknown random distribution, in contrast to prior studies that typically consider known, deterministic penalties~\cite{Bacinoglu2019finiteJournal,tripathi2019whittle}, known, random age-dependent rewards~\cite{chakraborty2025pilot}, or unknown adversarial penalties~\cite{Tripathi2021AoI,tripathi2021onlinearXiv}. In addition, our problem generalizes to age-dependent update cost, whereas constant update cost is commonly considered in the literature, ~\cite{Tripathi2021AoI,tripathi2021onlinearXiv}. Our work is closely related to the single-source, bandit-feedback framework explored in~\cite{Tripathi2021AoI}, where online decision-making is studied under adversarial penalty functions for information aging, using algorithms such as Follow the Perturbed Leader (FTPL) and EXP3. \Qiulin{Another line of work studies AoI bandits, e.g.,~\cite{Fatale2022AoI}, where
the arms are channels with unknown success probabilities and the age penalty is
deterministic.} In contrast, our study considers model degradation with an unknown, age-dependent distribution and further investigates how incorporating the specific problem structures can yield improved, near-optimal regret bounds.

\section{Model and Problem Formulation}\label{sec:model}

We consider a general scenario where a system operator maintains a deployed model in a dynamic environment, such as operating a machine learning inference model, a digital twin, or a remote monitoring system in volatile edge and IoT environments. As the underlying physical environment or data distribution evolves, the deployed model becomes stale, and the operator must determine the optimal schedule for retraining or synchronizing it. Our goal is to minimize the total cost, which comprises deterministic update costs (e.g., data acquisition and computation) and the stochastic performance degradation costs due to model staleness. We discuss the optimal solution with known degradation patterns and the performance metrics. 


\subsection{System Model}

We consider a slotted time horizon with $T$ slots. The deployed model is implemented at the beginning of the first slot. Similarly, we consider that updates of the model (if any) are conducted at the end of the slot and become available at the beginning of the next slot. An illustration of our system model is provided in Fig.~\ref{fig:AoI}.

\textbf{System Age and Freshness.}
To quantify staleness, we denote the age of the model at slot \(t\) as \(a_t\). The age evolves as follows.  It is initialized to one at slot \(1\). At any subsequent slot \(t\), if no update was performed at the end of the previous slot, the age increments by one. If an update was performed, the age resets to one. \footnote{Note that setting the model age to start from one is for the convenience of presentation. Our results remain valid when the age is initialized to zero.} Consequently, the age evolution is given by, 
\begin{equation}
   \hspace*{-5pt} a_t =\begin{cases}
        1, & \text{if $t=1$ or the model is updated at $t-1$;}\\
        a_{t-1}+1,& \text{otherwise}.\\
        
    \end{cases}
\end{equation}
We note that a smaller $a_t$ indicates a fresher model. In our problem, we consider that there is a minimum requirement on the model update, i.e., the model cannot stay un-updated for more than $K$ slots to guarantee a minimum level of freshness. In such a case, we have $a_t\leq K,\forall t\in[T]$. We note that our definition of model age is similar to the Age of Information (AoI) concept in communication networks~\cite{kaul2012real,yates2021age}; see more discussion in Sec.~\ref{sec:related-work}.

\begin{figure}[!t]
    \centering
    \includegraphics[width=.75             
    \columnwidth]{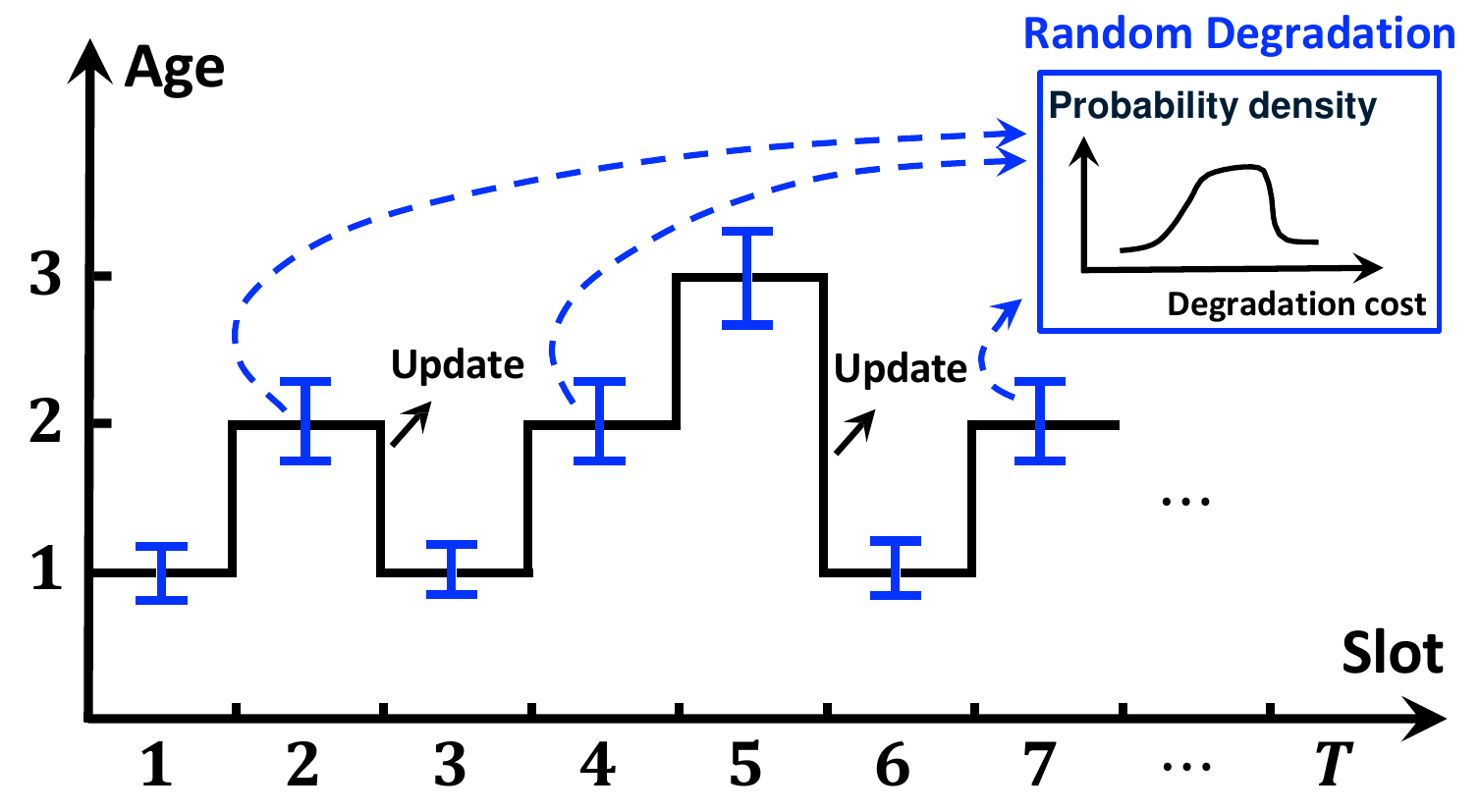}
    \caption{An illustration of the evolution of model age and the associated degradation cost. At each slot, the model age increments by one if no update is performed and resets to one upon an update. It incurs a random degradation cost following the probability distribution $f_a$ at age $a$. }
    \label{fig:AoI}
    \vspace{-2mm}
\end{figure}

\textbf{Degradation Cost.} Let $\mathcal{M}_{\tau}$ denote the version of the deployed model updated at slot $\tau$. At slot $t$, if the model’s age is $a_t$, the model being used at slot $t$ is the one updated at $t-a_t+1$, i.e., $\mathcal{M}_{t-a_t+1}$. We define the model degradation cost as the performance loss of applying $\mathcal{M}_{t-a_t+1}$ at slot $t$ and denote it as $g_t(\mathcal{M}_{t-a_t+1})$. In practice, performance degradation is driven by volatile and often unpredictable factors, such as sudden shifts in user behavior, environmental noise, or random state transitions in a physical system. To capture this inherent uncertainty, we model the degradation cost as a random variable. For ease of presentation, we simplify the notation to \(g_t(a_t)\).  
We assume that the degradation costs at different slots are independent and their distributions depend strictly on the model’s age. We let $f_a$ denote the probability density function (PDF) of the degradation cost for a given age $a$. Thus,  $g_t(a_t)$ is randomly generated from the distribution $ f_a$.  We assume that the degradation cost is upper bounded by a maximum constant $M$, i.e., $g_t(a_t)\leq M, \forall t, a_t$. We denote the expected value of $g_t(a_t)$ as $ \tilde{g}(a_t)$. 

\Qiulin{
\begin{remark}\label{remark:assumptions}
Our model assumes i.i.d.\ degradation costs for a given age and independent costs across slots. While these assumptions may not capture all non-stationary drift, they provide a tractable and non-trivial foundation for algorithm design. Here, we discuss the applicability of the model and where it is an approximation. The independence assumption is plausible in systems serving independent requests, where the degradation noise stems primarily from the variance of individual request feedback. Moreover, the update interval materially affects performance in such systems~\cite{Lee2023Recommender}, which supports modeling the degradation cost as a function of age. The age-based cost structure is also natural in remote monitoring of temporally correlated processes, where independence does not hold exactly; examples include remote estimation of a Wiener process~\cite{sun2019sampling} and the power-grid frequency trace of our empirical study \ifx \ISTR \undefined (Sec.~7 of the technical report~\cite{TechnicalReport}),
\else
(Sec.~\ref{sec:empirical}),
\fi
 on which the algorithms remain effective. The model is an approximation when the degradation is driven by observed context (e.g., ML retraining, digital twins) or the feasible action set by resource availability, rather than by age alone. These cases would require extensions, including a contextual-bandit formulation with observable context features and a constrained formulation with resource-dependent action sets. Further, time-varying degradation distributions could be handled by building on non-stationary~\cite{besbes2014stochastic} or rising-bandit~\cite{Rising2022Metelli} methods. We leave this for future work.
\end{remark}
}


\textbf{System Updating Cost.} Updating the deployed model incurs an explicit system cost, which may encompass data acquisition, computational overhead, and network transmission. We define \(C(a)\) as the cost to update an asset that currently has an age of \(a\). Because updates occur at the end of a slot, an update at slot \(t\) operates on an asset of age \(a_t\), incurring a cost of \(C(a_t)\). Formulating the cost as a function of age accommodates practical scenarios where updating a heavily degraded system requires more extensive data collection or longer retraining times, meaning \(C(a)\) may increase with \(a\).  It also generalizes from existing studies that consider a constant update cost, e.g.,~\cite{tripathi2019whittle,tripathi2021onlinearXiv}.

\begin{remark}\label{remark:update_cost} Unlike the degradation cost, which must be learned online, we assume the update cost $C(a)$ is known \emph{a priori}: update expenses are dominated by deterministic system-level resources (radio energy, bandwidth tariffs, cloud fees) that a designer can profile offline. \end{remark}

\subsection{Problem Formulation} \label{sec:formulation}

We aim to find an optimal update policy that minimizes the expected total cost, encompassing both the update cost and the degradation cost. At the beginning of slot \(t \in [T]\), the system operates with a model of age \(a_t\) and incurs a realized degradation cost \(g_t(a_t)\). At the end of slot \(t\), the decision-maker observes the history \(\mathcal{H}_t = \{a_1, \dots, a_t, x_1, \dots, x_{t-1}, g_1, \dots, g_t\}\) and makes an update decision \(x_t \in \mathcal{X}(a_t)\subseteq \{0,1\}\), where \(x_t=1\) denotes triggering an update and \(x_t=0\) denotes no update.  To meet the minimum requirement on the model update, we ensure the age does not exceed the maximum limit, \(K\). Then, the available action space depends on the current age: \(\mathcal{X}(a_t) = \{0, 1\}\) if \(a_t < K\), and \(\{1\}\) if \(a_t = K\).
A policy \(\pi = (\mu_1, \dots, \mu_T)\) is a sequence of decision rules where \(\mu_t\) maps the history \(\mathcal{H}_t\) to a valid action \(x_t \in \mathcal{X}(a_t)\). Let \(\Pi\) denote the set of all such non-anticipatory (causal) policies. The optimization problem is formulated as follows:

\begin{align}
\textsf{USP}:\quad \min_{\pi \in \Pi} \quad & \mathbb{E}^{\pi} \left[ \sum_{t=1}^{T} \left( g_t(a_t) + C(a_t) \cdot x_t \right) \right] \\
\text{s.t.} \quad & a_{t+1} = \begin{cases} 
                       1, & \text{if } x_t = 1,  \\
                       a_t + 1,  & \text{if } x_t = 0, \end{cases} \quad \forall t \in [T-1], \\
& a_1 = 1.
\end{align}

We consider that our problem starts from a freshly updated model ($a_1=1$). If the underlying distributions of the degradation costs were known \emph{a priori}, the Update Scheduling Problem (\textsf{USP}) could be solved optimally via standard dynamic programming. Under such known dynamics, we can also characterize the optimal policy for the asymptotic average cost objective, which exhibits a clear and highly interpretable structure. In practice, however, the operator rarely has access to these true degradation distributions beforehand. This motivates the core focus of our work: a challenging \emph{online scenario} where the decision-maker must simultaneously learn the unknown degradation profile while making real-time, cost-efficient update decisions. 

\subsection{Optimal Solution and Performance Metrics}

\textbf{Optimal Solution.} We consider minimizing the asymptotic average total cost defined as
\begin{equation} \label{eq:asym.avg.cots.minimization}
\limsup_{T\to\infty}\frac{1}{T}\,
\mathbb{E}^\pi\!\left[\sum_{t=1}^{T}\Big(g_t(a_t)+C(a_{t})x_t\Big)\right]
\end{equation}

We derive the optimal solution when the information of the expected degradation cost $\tilde{g}(a_t), \forall a_t\leq K$, is available.
\begin{proposition}\label{thm:offline}
    Given the expected degradation cost $\tilde{g}(a_t), \forall a_t\leq K$, the optimal solution to the problem in \eqref{eq:asym.avg.cots.minimization} is to update the learning model periodically with the optimal period $k^*$:
    \begin{equation}\label{eq:optimal_inter_update_time}
        k^*=\arg\min_{k\in[K]} \frac{\sum_{j=1}^{k}\tilde{g}(j)+C(k)}{k}.
    \end{equation} 
   And the optimal average cost is given by $\frac{\sum_{j=1}^{k^*}\tilde{g}(j)+C(k^*)}{k^*}$.
\end{proposition}

Because the degradation cost is independent across slots and does not affect system dynamics, the optimal policy depends only on age; the system reduces to a Renewal Reward process, and the Renewal Reward Theorem gives~\eqref{eq:optimal_inter_update_time}. It also suggests that an efficient online algorithm should aim to find the optimal update period. Please refer to 
\ifx \ISTR \undefined
our technical report \cite{TechnicalReport} for a detailed proof.
\else
Appendix~\ref{App:proof_of_offline}. 
\fi
\Qiulin{We note that Proposition~\ref{thm:offline} relates to Theorem~4 in~\cite{shisher2024timely},
which optimizes over policies that decide based on the age alone: under our
independence assumptions, the degradation history is decision-irrelevant, so our
age-based optimum agrees with theirs once we take the per-slot penalty as
$\tilde{g}(a)+C(a)-C(a-1)$.}

\textbf{Regret Minimization.} We use regret as a performance metric for evaluating the performance of an online policy. We denote the total cost of the optimal solution in Proposition~\ref{thm:offline} as $OPT(T)$ for a given period length $T$. Given an online algorithm, we denote its expected total cost as $ALG(T)$. We define the regret $R(T)$ as the difference between these two total costs, 
\begin{equation}
        R(T) = ALG(T) - OPT(T).
\end{equation}

In regret analysis, our goal is to develop online algorithms with low regret. Specifically, a sublinear regret ensures that the long-term average cost of the online algorithm converges to the same value as the optimal solution; see e.g.,~\cite{lattimore2020bandit}.



\section{A Novel Multi-Armed Bandit Variant}\label{sec:mab}

In this section, we first show that our problem can be viewed as a new variant of the multi-armed bandit (MAB) problem; we call it \emph{timing bandit}. We then discuss its problem structures. 

\subsection{Connecting \textsf{USP} with Multi-armed Bandit}

Following the optimal offline solution in~Proposition~\ref{thm:offline}, to achieve a lower regret in the online scenario, one should find the optimal update interval $k^*$ (the time between two successive updates) efficiently using the observed degradation costs from online update decisions. This follows a similar spirit to the multi-armed bandit problem, where we aim to learn the optimal arm with low regret. We connect the~\textsf{USP} problem to MAB by treating each update interval $k\in[K]$ as an arm. We define the average cost of arm $k$ as the average cost of choosing an update interval $k$ as

\begin{equation}\label{eq:cost_of_arm}
\mu_k\triangleq\frac{\sum_{t=1}^k \tilde{g}(i)+C(k)}{k}.
\end{equation} 

With the above definition, suppose we choose an update interval $k$, the average degradation cost and the update cost during these $k$ slots will be $k\cdot\mu_k$. Further, we denote $n_{k}(T)$ as the expected number of times that the update interval $k$ occurs in running an online algorithm. Then, we can show the following approximation of the regret, which is easy to interpret under the multi-armed bandit paradigm: 
\begin{equation}\label{eq:regret_alter}
    \tilde{R}(T) \triangleq \sum_{k=1}^{K} k\cdot \mu_k \cdot n_{k}(T) - T \cdot \mu_{k^*}.
\end{equation}

We give the bounds of this approximation in the following Lemma, which is due to the rounding issue of a finite $T$.
\begin{lemma}[Approximation of regret]\label{thm:regret-redefination}
The difference between the regret $R(T)$ and $\tilde{R}(T)$ is given by
\begin{equation}
    \left| R(T) - \tilde{R}(T) \right| \leq  K\cdot M.
\end{equation}
\end{lemma}

Lemma~\ref{thm:regret-redefination} implies that an online algorithm achieves the same order of regret bounds under $R(T)$ and $\tilde{R}(T)$. We thus focus on minimizing $\tilde{R}(T)$ in the following discussion. 

\subsection{Timing Bandit and Problem Structures} \label{sec:t-mab}

\begin{figure}[!t]
    \centering
    \includegraphics[width=.8\columnwidth]{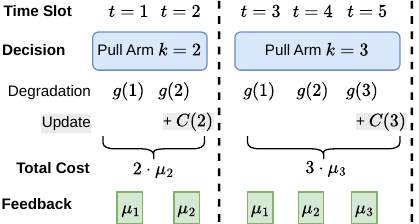}
    \caption{An illustrative example of timing bandit. Suppose we first pull arm $2$. We will observe the degradation cost of age $1$ and age $2$ in the following two slots, from which we can derive feedback on arm $1$ and arm $2$. The cost of such a pull includes the degradation cost and the update cost (at age $2$), the expectation of which equals $2\cdot\mu_2$ (according to~\eqref{eq:cost_of_arm}).  }
    \label{fig:tb}
    \vspace{-3mm}
\end{figure}

We are now ready to relate our problem to the multi-armed bandit paradigm with specific structures; we call it \emph{Timing Bandit}, referring to making timing update decisions.  We consider each update interval $k\in[K]$ as an arm in a $K$-arm multi-armed bandit problem. Each arm $k$ is associated with a random cost defined by $(\sum_{i=1}^{k}g(i)+C(k))/k$, where $g(i)\sim f_i$, and the mean cost of the arm is $\mu_k$ (defined in~\eqref{eq:cost_of_arm}) for all $k\in[K]$. {We provide an illustrative example in Fig.~\ref{fig:tb}.} The \emph{timing bandit} paradigm differs from the standard MAB problem in its decision-making, feedback structure, and specific arm cost composition, as detailed below:
\begin{enumerate} 
    \item Timing decision-making: in standard MAB, the decision is made at each time slot $t$ and incurs a single random cost associated with the chosen arm. In contrast, in the timing bandit, pulling arm $k$ would occupy the next $k$ time slots and would incur $k \cdot \mu_k$, i.e., $k$ times the expected cost of the arm.
    \item Arm cost composition: In timing bandit, the cost of arm $k$ is defined as the sum of two components: the average of $k$ random degradation costs, $g(i)$ from age $1$ to age $k$, and a fraction of the update cost, $C(k)/k$. This contrasts with standard MAB, where the cost of an arm is given by a single random variable associated with it.
    \item Consecutive feedback: in standard MAB, we receive feedback immediately after pulling an arm, which is a sample of the random cost associated with the arm. In the timing bandit, pulling arm $k$ yields feedback from arm $1$ to arm $k$ consecutively over the next $k$ rounds.
    
\end{enumerate}




Intuitively, the timing decision-making property makes the timing bandit problem more challenging, since pulling one arm would occupy more time slots and incur more cost, while only giving one sample of the pulled arm. In contrast, the second property turns out to allow a tighter confidence bound for estimating the arm cost that facilitates the learning, as we will discuss in Lemma~\ref{thm:more_information}.  The third property can facilitate learning by providing richer feedback from a single pulling. 

In the following, we characterize the challenge due to the timing decision-making structure, which also reveals the regret bound of directly applying the standard MAB algorithm to our problem without utilizing the latter two properties. 

\begin{proposition}
\label{thm:structure_lower_bound}
Consider a multi-armed bandit problem with $K$ arms and random costs with unknown means $\mu_k,\forall k\in[K]$, where pulling an arm (say arm $k$) will occupy $k$ slots, incur $k$ times its cost, and receive one sample on its cost. Then, any online algorithm for the problem will incur an expected regret of at least $\Omega(K\sqrt{T}/\log K)$. 
\end{proposition}

The proof follows the lower bound proof for the standard MAB (see Theorem 5.1 in~\cite{auer2002nonstochastic}) but with careful refinement to fit our case. Proposition~\ref{thm:structure_lower_bound} shows that the timing decision property increases the regret lower bound by a multiplicative factor of $\sqrt{K}/\log K$ compared to the standard MAB setting. This indicates that maintaining low regret is fundamentally more challenging for any online algorithm under this property. However, as we will show, the other two properties can help mitigate this difficulty through appropriate algorithmic design. 




\begin{remark}\label{remark:aoi_result} Our mathematical formulation is related to the time-varying cost of AoI optimization studied in~\cite{Tripathi2021AoI} (specifically, the single-sourcing case with bandit feedback, which is shown in the long version at arXiv~\cite{tripathi2021onlinearXiv}). The authors consider the adversary cost setting and apply the standard EXP3 algorithm. We note that directly applying their result to our setting will incur a regret of $\tilde{O}(K\sqrt{T})$, which is consistent with our finding in Proposition~\ref{thm:structure_lower_bound}; 
\ifx \ISTR \undefined
{see more discussion in our technical report~\cite{TechnicalReport}}.
\else
see more discussion in Appendix~\ref{app:compareTAoI}. 
\fi
We will show an improved regret by exploring the problem structures next.



\end{remark}

\section{Algorithm and Regret Analysis}\label{sec:algorithm}


In this section, we introduce an online algorithm, \textsf{BCAE}, that leverages the structures of the timing bandit. We establish that \textsf{BCAE} achieves a regret of $\tilde{O}(\sqrt{T})$. Additionally, we present a matching regret lower bound of $\Omega(\sqrt{T})$. We further propose an Optimism-Enhanced variant, \textsf{OE-BCAE}, that integrates lower confidence bound principles to improve empirical adaptivity while preserving the same regret order.

\subsection{Our Proposed Algorithm: \textsf{BCAE}}\label{sec:alg_dis}


Our algorithm, namely Balanced Consecutive Arm Elimination (\textsf{BCAE}), integrates the arm elimination strategy from the standard stochastic MAB literature (see e.g., a book of MAB~\cite{slivkins2019introduction}) while exploring the special structures of the timing bandit. Specifically, our algorithm introduces a tighter confidence bound by the arm cost composition and a balanced consecutive feedback idea to efficiently explore the consecutive feedback and the tighter confidence bound. We provide a pseudocode of the algorithm in Algorithm~\ref{alg:ae_independent}. 

In our \textsf{BCAE} algorithm, we maintain a set $\mathcal{A}$ of candidate arms that includes the optimal arm $k^*$ with high probability. After making the last update, we check our cost estimation of the arms in the candidate set $\mathcal{A}$. If the cost of an arm is far away from the optimal one we observed so far, we eliminate such an arm from $\mathcal{A}$. We note that we use $n_k(t)$ to denote the number of times arm $k$ is pulled before slot $t$. We use $m_k(t)$ to denote the number of samples for arm $k$ before slot $t$. For ease of the presentation, we simplify them to be $n_k$ and $m_k$ in Algorithm~\ref{alg:ae_independent}, which represent the latest value of $n_k(t)$ and $m_k(t)$. Similarly, we use $\bar{\mu}_k$ to denote the latest estimation of $\mu_k$. Our algorithm introduces the following new designs.

\textbf{Tighter Confidence Bound.} Following the arm cost composition property in Sec.~\ref{sec:t-mab}, with $m_k(t)$ samples of $\mu_k$ (each including samples of $g(a),\forall a\in[k]$), $\bar{\mu}_k$ becomes that average of $k\cdot m_i(t)$ bounded, independent random variables. According to Hoeffding Inequality, it allows us to derive a tighter confidence bound for the estimation of $\mu_k$, i.e., $|\mu_k-\bar{\mu}_k|\leq  \sqrt{\frac{M^2\log (\sqrt{K}T)}{k\cdot m_k}}$ with high probability, compared with the standard bound (which replaces $k\cdot m_k$ by $m_k$). We formalize it in Lemma~\ref{thm:more_information} in the regret analysis. It leads to a refined condition for arm elimination, Line~\ref{line:thresholding_b} in Algorithm~\ref{alg:ae_independent}. 

\textbf{Balanced Confidence Gaps with Consecutive Feedback.} To explore the consecutive feedback, existing literature, e.g.,~\cite{Swapna2018Side,zhao2019,Li2021Waiting}, proposed the idea to pull the arm with the maximum index in the candidate set $\mathcal{A}$, thereby efficiently collecting samples for all arms in $\mathcal{A}$ with a single pull. While it seems natural to combine this maximum-index-pulling idea with the tighter confidence bounds, this direct combination is insufficient to significantly reduce regret. This is because in arm elimination, the elimination condition (see Line~\ref{line:thresholding_b} in Algorithm~\ref{alg:ae_independent})
relies on shrinking both confidence bounds of the compared arm pairs. The maximum-index-pulling method only ensures the same number of samples, $m_i,\forall i\in\mathcal{A}$. The arms with smaller indices have looser confidence bounds, compared with larger ones (due to the $i\cdot m_i$ term in the denominator), and will dominate the number of samples needed to shrink the confidence bounds to a certain level, leading to a limited utilization of the tighter confidence bound. 


\begin{figure}[!t]
    \centering
    \includegraphics[width=.9\columnwidth]{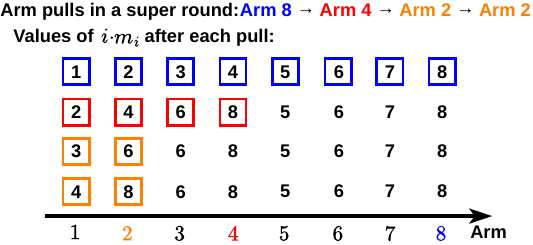}
    \caption{An illustrative example of \textsf{BCAE} in a super-round. Suppose there are 8 arms. Then, the anchor arms are $\{2,4,8\}$. Each square shape represents an observed sample; e.g., pulling arm 8, we have samples from all arms. The numbers represent the values of $i\cdot m_i$ of the arms after each pull. We see that at the end of the super-round, the values of $i\cdot m_i$ are within a factor of two of each other, which leads to a balanced confidence interval among all the arms. }
    \label{fig:BCAE}
    \vspace{-2mm}
\end{figure}

Instead, we introduce the idea of strategically balancing the confidence gaps,  $\sqrt{\frac{M^2\log (\sqrt{K}T)}{i\cdot m_i}}$, for all arms, while efficiently utilizing the consecutive feedback. {We provide an illustration in Fig.~\ref{fig:BCAE}.}  We introduce the concept of super-round, which includes a number of consecutive arm pullings. It corresponds to the period of running Line~\ref{line:arm_selection_b} to Line~\ref{line:end_super} for one iteration.  Our design is to balance the confidence gaps of all arms at the end of each super-round, i.e., ensuring that $i\cdot m_i,\forall i\in\mathcal{A}$ are close to each other. Specifically, at a super-round, we identify arms with index $2^{\hat{j}-\tilde{j}},\forall \tilde{j}\in\{0,1,\cdots \hat{j}-1\}$, where  $\hat{j}=\lceil\log_2 \hat{k}\rceil$ and $\hat{k}$ is the maximum index in $\mathcal{A}$. We call these arms anchors. For each anchor arm $ 2^{\hat{j}-\tilde{j}}, \tilde{j}\in\{0,1,\cdots \hat{j}-1\}$, we pull it $\hat{z}=2^{\tilde{j}}-2^{\tilde{j}-1}$ times (or one time when $\tilde{j}=0$). In case the anchor arm has been eliminated, we delegate the updates to the closest arm to it from the left, $\tilde{i}:=\max_i\{i\in\mathcal{A} | i\leq 2^{\hat{j}-\tilde{j}}\}$. 

For illustrative purposes, we assume that all anchor arms remain in $\mathcal{A}$. For each anchor arm $2^{\hat{j}-\tilde{j}}$, we obtain $1+\sum_{{j}=1}^{\tilde{j}}(2^{{j}}-2^{{j}-1})=2^{\tilde{j}}$ samples, either by pulling it or anchor arms with larger index (exploring the consecutive feedback). It ensures the increment of $i\cdot m_i$ for all anchor arms $i$ are similar, i.e., $2^{\hat{j}-\tilde{j}}\cdot 2^{\tilde{j}}=2^{\hat{j}}\simeq\hat{k}, \forall \tilde{j}$. For each other arm, it obtains the same number of samples in the super-round as the anchor arm with the smallest index greater than its own, due to the consecutive feedback structure. And as its index is at least half of such an anchor arm, its increment in $i\cdot m_i$ is also at least half of such an anchor arm. We provide a formal statement of this property in Lemma~\ref{thm:comparable_c_i} within our regret analysis. 


For the rest of the algorithm, after each arm pull, we update our estimates for arms with indices smaller than $\tilde{i}$. We count the number of times the arm is pulled, $n_{\tilde{i}}$, and the number of samples we obtained for the arms, $m_{i},\forall i\leq \tilde{i}\; \text{and}\; i\in\mathcal{A}$. We apply the arm elimination condition to eliminate non-optimal arms.  These correspond to Line~\ref{line:arm_pulling_b_b} to Line~\ref{line:first_end_b_b} in Algorithm~\ref{alg:ae_independent}. We repeat the process until the end of the horizon $T$.

\begin{algorithm}[!t]
\caption{\mbox{Balanced Consecutive Arm Elimination \textsf{(BCAE)}}\label{alg:ae_independent}}
\begin{algorithmic}[1]
\STATE Keep the model for $K$ slots, receive degradation cost $g_t(a_t)$ with $a_t = t$, $\forall t\in[K]$. Perform an update at the end of slot $K$. 
\STATE Initialize $\bar{\mu}_{k}=\frac{1}{k}\left(C(k)+\sum_{i=1}^{k}g_{i}(i)\right)$.
\STATE Initialize $\mathcal{A}=\{1,2,\cdots,K\}$.
\STATE $t=K$ and $m_{k}=1$, $\forall k\in[K]$.
\STATE $j^*=\arg\min_{i\in{\mathcal{A}}} \bar{\mu}_{i}+\sqrt{\frac{M^2\log (\sqrt{K}T)}{i\cdot m_i}}$. \label{line:baseline_1}
\FOR{$i\in \mathcal{A}$}
\STATE If {$\bar{\mu}_{i}-\sqrt{\frac{M^2\log (\sqrt{K}T)}{i\cdot m_i}}> \bar{\mu}_{j^*}+ \sqrt{\frac{M^2\log (\sqrt{K}T)}{j^*\cdot m_{j^*}}}$}, $\mathcal{A}= \mathcal{A} - \{i\}$. \label{line:thresholding}
\ENDFOR \label{line:first_end_b}
\WHILE{ $t < T$}
\STATE Let $\hat{k}=\max_k \{k\in\mathcal{A}\}$, and $\hat{j}=\lceil\log_2 \hat{k}\rceil$. \label{line:arm_selection_b}
\STATE If $\hat{j}=0$, Update the model at the end of the slot; $t=t+1$; Continuous.
\FOR{$\tilde{j}=0:1:\hat{j}-1$}
\STATE Let $\tilde{i}=\max_i\{i\in\mathcal{A} | i\leq 2^{\hat{j}-\tilde{j}}\}$. If no such $\tilde{i}$, BREAK.\label{line:arm_pulling_b}
\STATE Let $\hat{z}=1$ if $\tilde{j}= 0$; otherwise, $\hat{z}=2^{\tilde{j}}-2^{\tilde{j}-1}$.
\FOR{$z=1:1:\hat{z}$}
\STATE Let $\tilde{i}=\max_i\{i\in\mathcal{A} | i\leq 2^{\hat{j}-\tilde{j}}\}$. If no such $\tilde{i}$, BREAK.\label{line:arm_pulling_b_b}
\STATE If $t+\tilde{i} > T$, Keep the model till $T$ and EXIT.
\STATE Pull arm $\tilde{i}$, i.e., Keep the model for $\tilde{i}$ slots, observe   $g_{t+1}(1),g_{t+2}(2)\cdots g_{t+\tilde{i}}(\tilde{i})$, and  update the model.
\STATE Update $t=t+\tilde{i}$; Update $n_{\tilde{i}}=n_{\tilde{i}}+1$.
\FOR{$i\in \mathcal{A}$ and $i\leq \tilde{i}$}
\STATE Obtain a sample of $\mu_{i}$, i.e., let  $\hat{\mu}_{i}=\frac{1}{i}\left(C(i)+\sum_{k=1}^{i}g_{t+k}(k)\right)$;\label{line:update_mu}
\STATE Update $\bar{\mu}_{i}=\frac{m_{i}\cdot\bar{\mu}_{i}+\hat{\mu}_{i}}{m_{i}+1}$; Update $m_{i}= m_{i}+1$.
\ENDFOR
\STATE $j^*=\arg\min_{i\in{\mathcal{A}}} \bar{\mu}_{i}+\sqrt{\frac{M^2\log (\sqrt{K}T)}{i\cdot m_i}}$. \label{line:baseline}
\FOR{$i\in \mathcal{A}$}
\STATE If {$\bar{\mu}_{i}-\sqrt{\frac{M^2\log (\sqrt{K}T)}{i\cdot m_i}}> \bar{\mu}_{j^*}+ \sqrt{\frac{M^2\log (\sqrt{K}T)}{j^*\cdot m_{j^*}}}$}, $\mathcal{A}= \mathcal{A} - \{i\}$. \label{line:thresholding_b}
\ENDFOR \label{line:first_end_b_b}
\ENDFOR

\ENDFOR \label{line:end_super}
\ENDWHILE
\end{algorithmic}

\end{algorithm}

\subsection{Regret Analysis}\label{sec:regret_analysis}

Here, we illustrate the key properties of our proposed \textsf{BCAE} algorithm. We provide a theoretical analysis of its performance on regret minimization. Furthermore, we establish a corresponding regret lower bound for the~\textsf{USP} problem, demonstrating the near-optimality of our approach.  
\ifx \ISTR \undefined
Due to space limitations, we provide the main proof ideas here. A full version with proofs is  in~\cite{TechnicalReport}.
\else

\fi


We first formally describe the tighter confidence bound we applied in our algorithm~\textsf{BCAE} and the balanced confidence gaps guaranteed by our designs in the algorithm.


\begin{lemma}[Tighter Confidence Bound]\label{thm:more_information}
\begin{equation}\label{eq:confidence_bound}
   \mathbf{P}\left(|\bar{\mu}_k(t)-\mu_k|\geq \sqrt{\frac{M^2\log (\sqrt{K}T)}{k\cdot m_k(t)}}\right)\leq 2\cdot K^{-1}T^{-2}.
\end{equation}
\end{lemma}
\ifx \ISTR \undefined

\else
The details of the proof are in Appendix~\ref{app:thm:more_information}. 
\fi
The proof is based on the Hoeffding Inequality~\cite{Hoeffding1963Inequality} and the specific arm-cost composition of our problem. In analysis of the standard MAB, e.g.,~\cite{lattimore2020bandit}, the confidence gap, usually defined as $\sqrt{\frac{M^2\log (\sqrt{K}T)}{m_k(t)}}$, shrinks with the number of samples ${m_k(t)}$. Our result introduces a tighter confidence bound, which shrinks faster in the arm index.  

We denote $t_s$ as the beginning slot of the $s$-th super-round. 

\begin{lemma}\label{thm:comparable_c_i}
 After the previous super-round, we have that 
    \begin{equation}\label{eq:super-round}
        \frac{i\cdot m_i(t_s)}{j\cdot m_j(t_s)} \leq 2, \forall i,j\in\mathcal{A}(t_s).
    \end{equation}
    
\end{lemma}
The lemma implies that, for all arms in $\mathcal{A}(t_s)$, their confidence gaps, $\sqrt{\frac{M^2\log (\sqrt{K}T)}{i\cdot m_i(t_s)}},\forall i\in\mathcal{A}$, are tightly clustered, differing by no more than a multiplicative factor of $\sqrt{2}$. The lemma is due to our dedicated algorithmic design for balancing confidence gaps, which ensures that the increments of $i\cdot m_i$ in each super-round are close for all arms in the candidate set. \ifx \ISTR \undefined

\else
The details of the proof are in Appendix~\ref{app:thm:comparable_c_i}. 
\fi 

Suppose that the optimal arm is with index $k^*$, and we denote the difference between the expected cost of arm $k$ and arm $k^*$ as $\Delta_k=\mu_{k}-\mu_{k^*}$. We consider a clean event $\mathcal{E}$: 
\begin{equation}\label{eq:regret_alter_1}
    \mathcal{E}\triangleq 
    \left\{  |\bar{\mu}_k(t)-\mu_k|\leq \sqrt{\frac{M^2\log (\sqrt{K}T)}{k\cdot m_k(t)}},\forall t \in [T],k \in [K] \right\},
\end{equation}

We can show that event $\mathcal{E}$ happens with high probability based on Lemma~\ref{thm:more_information} and the union bound. It would be sufficient to show that our algorithm can obtain the desired regret bound under the clean event. We note that under event $\mathcal{E}$, the optimal arm $k^*$ will always be in $\mathcal{A}$, by observing that the condition in the algorithm \textsf{BCAE} to eliminate an arm from $\mathcal{A}$ will never hold for arm $k^*$.


We show that, under the clean event $\mathcal{E}$, algorithm \textsf{BCAE} only requires a bounded number of samples to eliminate an arm $i$ with optimality gap $\Delta_i$. 

\begin{lemma}\label{thm:bounded_sample}
   We have $m_i(T)\leq \frac{36M^2\log{(\sqrt{K}T)}}{i\cdot \Delta_i^2} + 2\cdot K, \forall i\neq k^*$.
\end{lemma}
\ifx \ISTR \undefined

\else
The details of the proof are in Appendix~\ref{app:thm:bounded_sample}. 
\fi
Unlike the standard bound, ours scales inversely with the arm index $i$ as well as $\Delta_i^2$, thanks to the tighter confidence bound (Lemma~\ref{thm:more_information}) and the balanced confidence gaps (Lemma~\ref{thm:comparable_c_i}). The idea of the proof is that after the required number of samples has been collected, our balancing idea ensures that the confidence bounds for the suboptimal arm $i$ and the optimal arm $k^*$ become sufficiently and consistently narrow, which triggers the elimination condition once the optimality gap $\Delta_i$ can be reliably distinguished. 


Our algorithm will pull arms with a smaller index to ensure balanced confidence gaps. Such behavior deviates from the maximum-index-pulling idea in the literature~\cite{Swapna2018Side,zhao2019,Li2021Waiting} and may not optimally utilize the consecutive feedback structure. For example, under the maximum-index-pulling idea, arms whose index values fall below that of the optimal arm are almost unlikely to be pulled since the optimal arm remains in the candidate set with high probability and will dominate the index comparison. As a comparison, our algorithm sometimes pulls these arms. Nevertheless, our novelty in the analysis is that we show that our algorithm is able to utilize the consecutive feedback efficiently (incurring a low regret) in each subset of a partition of the total $K$ arms. Specifically, the partition, induced by the anchor arms, consists of $\lceil\log_2 K\rceil$ subsets, which is denoted as $\{\mathcal{K}_{j}| j\in\lceil\log_2 K\rceil\}$, where $\mathcal{K}_{1}=\{1,2\}$ and $\mathcal{K}_{j}=\{2^{j-1}+1,\cdots,\min\{K,2^{j}\}\}$, for $j\in\{2,\cdots,\lceil\log_2 K\rceil\}$. We obtain the following upper bound on the regret incurred by pulling the arms in each subset.

\begin{lemma}\label{thm:subarms-regret}
    We have that, for any $j\in\{1,\cdots,\lceil\log_2 K\rceil\}$, 
    \begin{equation}\label{eq:subarms_regret}
    \sum_{i\in\mathcal{K}_j} i\cdot n_i(T)\cdot \Delta_i     
    \leq \sqrt{36TM^2\log (\sqrt{K}T)}\cdot L_1 + 2K^2(\bar{C}+M),
    \end{equation}
where $L_1\triangleq\left(2+\log\frac{2K(\bar{C}+M)\sqrt{T}}{\sqrt{36M^2\log (\sqrt{K}T)}}\right)$ and $\bar{C}\triangleq\max_k C(k)$.
\end{lemma}

\ifx \ISTR \undefined

\else
The details of the proof are in Appendix~\ref{app:thm:subarms-regret}. 
\fi
 Essentially, the lemma implies that the regret incurred by pulling suboptimal arms in each subset $\mathcal{K}_j$ is bounded by $\tilde{O}(\sqrt{T})$.  Together with the fact that the number of subsets in our partition is bounded by $\log_2 K+1$, we can conclude the regret upper bound of our algorithm, which will be shown in Theorem~\ref{thm:ae_independent}.

The lemma is based on the tighter bounds on the required number of samples shown in Lemma~\ref{thm:bounded_sample}. We further show that in each subset, an arm will not be pulled before all arms with larger indices (in the subset) are eliminated. It ensures the following inequality under the clean event $\mathcal{E}$,  $\forall k\in \mathcal{K}_{j}$,
\begin{equation}\label{eq:condition_number_of_pulls_app_b_m}
   n_k(T)=0\; \text{or}\; \textstyle\sum_{i=k}^{\min\{K, 2^j\}} n_i(T) \leq \frac{36M^2\log (\sqrt{K}T) }{k\cdot \Delta_k^2} + 2 K.
\end{equation}
That is, either that an arm $k$ will never be pulled (it is eliminated due to sufficient samples from pulling arms with larger indices) or that its sample bound provides an upper bound to the total number of pulls of all arms with no smaller indices (in the subset). The bound is tighter than the one in the standard MAB, where the number of pulls of each arm is bounded by its required sample size individually.  
The proof of the lemma is based on solving the optimization problem of maximizing the incurred regret subject to the constraints on the number of pulls of the arms (including~\eqref{eq:condition_number_of_pulls_app_b_m}). 

We now show the regret guarantee of our proposed algorithm in the following theorem, which is obtained by summing up the regret incurred by all the $\lceil\log_2 K\rceil$  subsets and Lemma~\ref{thm:subarms-regret}.

\begin{theorem}\label{thm:ae_independent}
    The regret of our \textsf{BCAE} algorithm is upper bounded by $
    \sqrt{36TM^2\log (\sqrt{K}T)}\cdot (\log_2 K +1)\cdot L_1 + 2K^2(\log_2 K +1)(\bar{C}+M),$ where $L_1$ and $\bar{C}$ are as those defined in Lemma~\ref{thm:subarms-regret}.
\end{theorem}
 



    

Theorem~\ref{thm:ae_independent} shows that our \textsf{BCAE} algorithm achieves a sublinear regret upper bound of $\tilde{O}(\sqrt{T})$ for sufficiently large $T$. Here, the $\tilde{O}$ notation suppresses logarithmic factors. We note that sublinear regret implies that the long-term average cost of our proposed algorithm converges to the same value as the optimal solution. Our approach could effectively learn to make optimal update decisions. In addition, it improves over simply applying standard MAB algorithms (achieving $\tilde{O}(K\sqrt{T})$, Proposition~\ref{thm:structure_lower_bound}). Finally, with the lower bound identified in the following Proposition~\ref{thm:lower_bound}, our algorithm achieves a regret bound close to the lower bound by up to logarithmic factors. 



\begin{proposition}\label{thm:lower_bound}
    The regret for any online algorithm for \textbf{USP} is lower bounded by $\Omega(\sqrt{T})$.
\end{proposition}

We prove the lower bound following the proof of the lower bound for stochastic multi-armed bandit with two arms; see Theorem 2.10 with $K=2$ in~\cite{slivkins2019introduction} for an illustration. We construct two instances of~\textsf{USP}, where $g_1$ follows a Bernoulli distribution with mean $1/4$, $C(1)=C(2)=1/4$, and $g_2$ follows a Bernoulli distribution with mean $1/2-\epsilon$ in the first instance and with mean $1/2+\epsilon$ in the second case. Under both instances, we can check that playing a suboptimal arm will incur a loss of $1/2\cdot \epsilon$ per round. Following the information-theoretic argument in~\cite{slivkins2019introduction}, it requires $\Omega({1}/{\epsilon^2})$ rounds to determine the optimal arms in both instances. It will incur a regret of $\Omega(\epsilon\cdot \min\{{1}/{\epsilon^2},T\})$, which by tuning $\epsilon=\Theta(1/\sqrt{T})$, leads to a regret lower bound of $\Omega(\sqrt{T}$).

\subsection{Optimism-Enhanced Balanced Consecutive Arm Elimination (\textsf{OE-BCAE})}\label{sec:OE-BCAE}

\begin{figure*}[!ht]
\subcaptionbox{Performance across the horizon $T$  \label{fig:T}}%
{\includegraphics[width=0.32\linewidth]{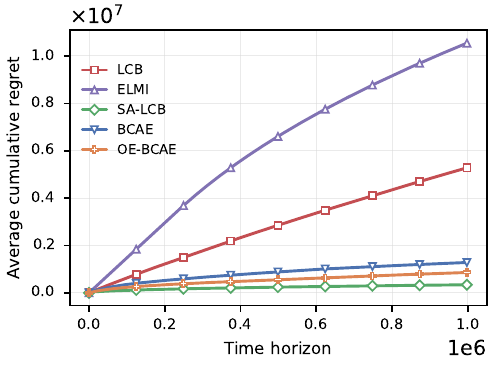}}
  \hspace{\fill}
  \subcaptionbox{Performance under different $K$\label{fig:K}}%
{\includegraphics[width=0.32\linewidth]{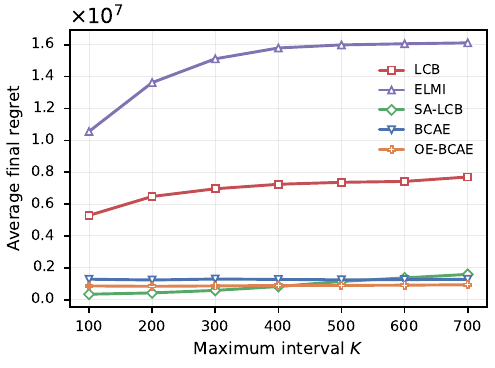}}
     \hspace{\fill} 
    \subcaptionbox{Performance under different $C_0$\label{fig:C}}%
{\includegraphics[width=0.32\linewidth]{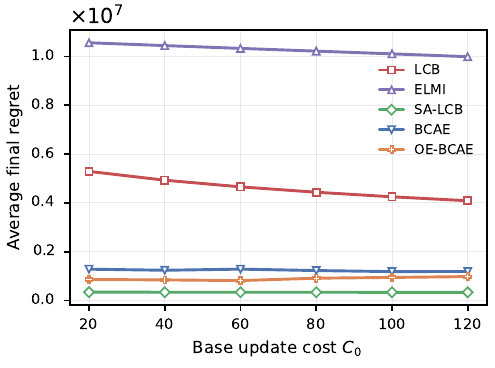}}
\caption{The performance of comparing algorithms under different parameters.}
  \label{fig:sample}
\vspace{-3mm}
\end{figure*}

 
Our proposed \textsf{BCAE} algorithm follows the arm elimination paradigm. It permanently discards an arm only when there is high confidence that the arm is suboptimal. Because mistaken elimination is irrecoverable, the criterion is necessarily conservative, so the algorithm adapts only at infrequent elimination events. In contrast, the Lower Confidence Bound (LCB) paradigm reselects the arm with the smallest LCB every round, making it more responsive to incoming feedback. Such a behavior of arm elimination may not be preferable from an empirical evaluation perspective, as observed in the literature, e.g., ~\cite{zhao2019,Swapna2018Side}. In this section, we incorporate the optimism idea in the lower confidence bound paradigm and propose Optimism-Enhanced Balanced Consecutive Arm Elimination (\textsf{OE-BCAE}). We show that \textsf{OE-BCAE} maintains the near-optimal regret bound. We will discuss numerical performance in Sec.~\ref{sec:simulation}.\footnote{A direct LCB-based design remains open (see \textsf{SA-LCB} in Sec.~\ref{sec:simulation}), as it is unclear how LCB variants can efficiently utilize the consecutive feedback structure; see e.g.~\cite{zhao2019,Swapna2018Side}.}


Our \textsf{OE-BCAE} algorithm makes the following modification to \textsf{BCAE}. At Step~\ref{line:arm_pulling_b_b} of Algorithm~\ref{alg:ae_independent}, instead of simply pulling the arm with the maximum index no larger than $2^{\hat{j}-\tilde{j}}$, we apply the lower confidence bound algorithm in a forward window from the arm. More specifically, for each $\tilde{i}$, we consider the forward window, 
\begin{equation}\label{eq:defn:window}
    \mathcal{W}_{\tilde{i}}\triangleq \{\tilde{i}\} \cup \left(\mathcal{A} \cap  [l+1, l+w]\cap [l, 2\cdot l]\right) ,
\end{equation}
where $l \triangleq 2^{\lceil \log_2 {\tilde{i}}\rceil}$ and $w$ is an upper bound on the window size that can be designed. And, we pull the arm
\begin{equation}\label{eq:LCB_in}
    \hat{i} = {\arg\min}_{i\in \mathcal{W}_{\tilde{i}}} \; \bar{\mu}_{i}-\sqrt{\frac{M^2\log (\sqrt{K}T)}{i\cdot m_i}}.
\end{equation}

The use of the forward window ($\hat{i}$ is no smaller than $ \tilde{i}$) allows us to efficiently utilize consecutive feedback, similar to pulling $\tilde{i}$: we receive all the feedback from arm one to arm $\tilde{i}$ at once. Restricting the arm indexes in the window to $2\cdot l$ can keep the confidence bounds among arms balanced. These two maintain the two properties of \textsf{BCAE} that are important for achieving the near-optimal regret guarantee. Further, by limiting the number of arms in the forward window, $w$, we can limit the additional regret or risk incurred compared with directly pulling arm $\tilde{i}$. All together, these allow us to show that \textsf{OE-BCAE} also achieves the near-optimal regret $\tilde{O}(\sqrt{T})$.

\begin{theorem}\label{thm:ae_independent_oe}
    The regret of our \textsf{OE-BCAE} algorithm is upper bounded by $\tilde{O}(\sqrt{T})$ when $w$ is a constant or ${O}(\log{K})$.  
\end{theorem}


\ifx \ISTR \undefined

\else
The details of the proof are in Appendix~\ref{app:thm:ae_independent_oe}. 
\fi
Our idea to prove the theorem is twofold. First, we decompose the regret into two parts, where the first part is induced as if we play $\tilde{i}$, and the second part is the additional regret from actually pulling $\hat{i}$ according to~\eqref{eq:LCB_in}. Second, we show the first part incurs a regret bounded by $\tilde{O}(\sqrt{T})$ following a similar analysis as that of the \textsf{BCAE} algorithm, and the second part functions analogously to LCB with a constant number of arms, contributes an additional regret term bounded by $\tilde{O}(\sqrt{T})$.


\section{Simulation Results}\label{sec:simulation}
In this section, we evaluate the performance of our proposed online learning algorithms through numerical simulations. Our evaluation aims to demonstrate the efficiency of our proposed algorithms and compare them with the alternatives. \ifx \ISTR \undefined
An additional empirical evaluation on a remote monitoring instance using real-world traces is provided in~\cite{TechnicalReport} due to space limitations.
\else
An additional empirical evaluation on a remote monitoring instance using real-world traces is provided in Sec.~\ref{sec:empirical}.
\fi  

\subsection{Simulation Setup}\label{sec:num_setup}

\textbf{Baselines for Comparison:} We compare our proposed algorithms against the following baselines:
\begin{itemize}
    \item \textsf{LCB:} The standard Lower Confidence Bound algorithm. 
    \item \textsf{ELMI}, e.g.,~\cite{Swapna2018Side,zhao2019,Li2021Waiting}: An arm-elimination variant that always pulls the maximum-index arm in the candidate set.
    \item \textsf{\textsf{SA-LCB}}, Structure-Aware LCB: An LCB variant we adapt regarding the problem structures, i.e., replacing the confidence bound used in LCB with the one we develop in Lemma~\ref{thm:more_information}. 
    \item \textsf{\textsf{BCAE}:} Our proposed Balanced Consecutive Arm Elimination algorithm in Algorithm~\ref{alg:ae_independent}.
    \item \textsf{\textsf{OE-BCAE}:} Our Optimism-Enhanced variant, which integrates LCB principles into \textsf{BCAE}, discussed in Sec.~\ref{sec:OE-BCAE}.
\end{itemize}

\textbf{Environment and Cost Models:} We simulate a dynamic system where an operator maintains a deployed model over a time horizon \(T\). The decision space consists of \(K\) candidate update intervals (arms).  The expected degradation cost is modeled as an increasing function of the model age \(a\), specifically \(\tilde{g}(a) = 50\cdot (a/K)^2\). The realized degradation cost at slot \(t\) is sampled from a truncated normal distribution with mean \(\tilde{g}(a_t)\) and a random standard deviation \(\sigma = \alpha \cdot \tilde{g}(a)\) where $\alpha$ is chosen as $0.2$. 
The deterministic system update cost is set to be \(C(a) = C_0 + 50\sqrt{a/K}\). The sublinear square root function captures diminishing marginal overhead for updating. {We set the degradation and update costs to be parameterized by the normalized age $a/K$. It ensures that increasing $K$ increases the number of available arms without proportionally enlarging the cost range, capturing a setting with more candidate intervals but comparable cost scale.} 

By default, we set the time horizon to \(T = 10^6\) slots, the maximum candidate interval to \(K = 100\), and the base update cost \(C_0 = 20\). Results are averaged over 30 independent simulation runs to ensure statistical confidence.


\subsection{Results and Discussion}

Fig.~\ref{fig:T} illustrates the cumulative regret of the evaluated algorithms under the default settings. In Fig.~\ref{fig:K}, we demonstrate the regret of the algorithms under different maximum intervals $K$. Fig.~\ref{fig:C} shows the regret of the algorithms under different base update costs $C_0$.

The results show that \textsf{BCAE} and \textsf{OE-BCAE} clearly improve over \textsf{LCB} and \textsf{ELMI}, which do not or only partially exploit the favorable problem structures, and remain robust across varying maximum intervals and base update costs, highlighting their scalability and adaptability. Compared with \textsf{SA-LCB}, a baseline we adapt from the standard LCB by incorporating the problem structures, \textsf{SA-LCB} achieves the lowest regret when the number of arms is small, whereas our algorithms, particularly \textsf{OE-BCAE}, are superior in larger search spaces. Although \textsf{SA-LCB}'s competitive performance underscores the value of the structural insights in this paper, its theoretical regret analysis remains an open challenge.

The plots also highlight the empirical advantage of \textsf{OE-BCAE} over \textsf{BCAE}: by integrating LCB-style optimism, \textsf{OE-BCAE} avoids the delayed adaptation inherent to pure elimination paradigms, resulting in a lower-regret curve.

\Qiulin{\textit{Robustness to temporal correlation.} We relax the independence
assumption (Remark~\ref{remark:assumptions}) by making the degradation noise
AR(1): $g_t(a_t)=\max(0,\tilde{g}(a_t)+\sigma(a_t)z_t)$ with
$z_t=\rho z_{t-1}+\sqrt{1-\rho^2}\,\eta_t$ ($\eta_t$ i.i.d.\ standard normal),
leaving the per-age mean and variance unchanged, so $\rho=0$ recovers the
i.i.d.\ model. Table~\ref{tab:rho} shows \textsf{OE-BCAE}'s mean regret is flat
across $\rho\in\{0,\dots,0.9\}$. Other baselines behave
likewise, but details are omitted due to space limit. Intuitively, correlation inflates the variance of the estimates without
shifting their means, while the confidence bounds of
Lemma~\ref{thm:more_information} are variance-free and far larger than the estimation noise. Hence, the arm-selection decisions and the mean regret barely change. We note that under correlation, the realized degradation predicts future
degradation, so an optimal policy should adapt to the realized degradation. Designing such policies is left to future
work.
}

\Qiulin{
\begin{table}[!t]
\caption{Average regret of \textsf{OE-BCAE}  versus the
autocorrelation coefficient $\rho$. }
\label{tab:rho}
\centering
\begin{tabular}{lrrrrrr}
\toprule
$\rho$ & 0 & 0.3 & 0.5 & 0.7 & 0.9 \\
\midrule
\textsf{OE-BCAE}( $\times 10^5$) & 8.62 & 8.63 & 8.62 & 8.62 & 8.62 \\
\bottomrule
\end{tabular}
\end{table}
}
\ifx \ISTR \undefined
\else

\begin{figure*}[!ht]
\subcaptionbox{Performance across the horizon $T$  \label{fig:e_T}}%
{\includegraphics[height=0.23\textwidth]{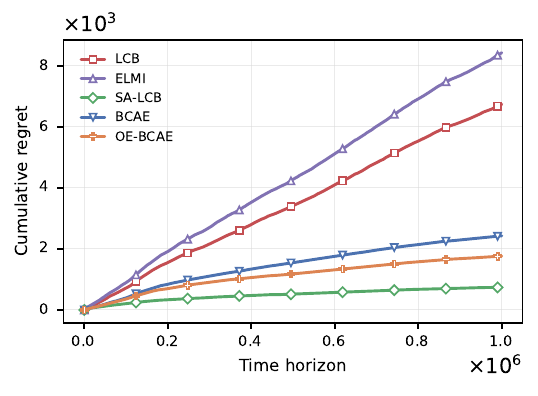}}
  \hspace{\fill}
  \subcaptionbox{Performance under different $K$\label{fig:e_K}}%
{\includegraphics[height=0.23\textwidth]{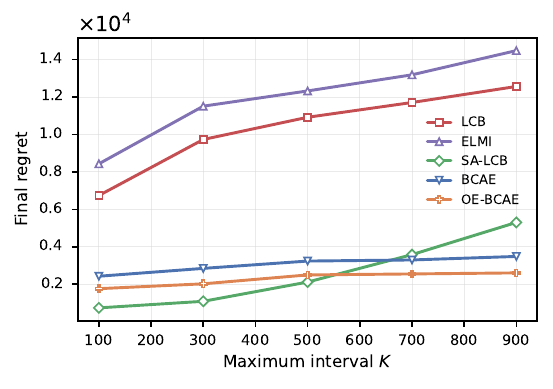}}
     \hspace{\fill} 
    \subcaptionbox{Performance under different $C$\label{fig:e_C}}%
{\includegraphics[height=0.23\textwidth]{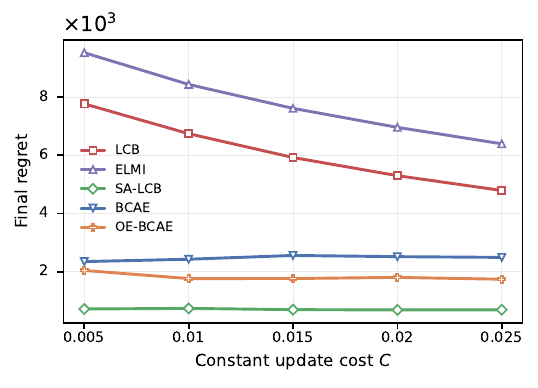}}
\caption{Empirical Evaluation on Real Frequency Trace.}
  \label{fig:empirical}
\vspace{-3mm}
\end{figure*}

\begin{figure}[!ht]
    \centering
    \includegraphics[width=0.75\columnwidth]{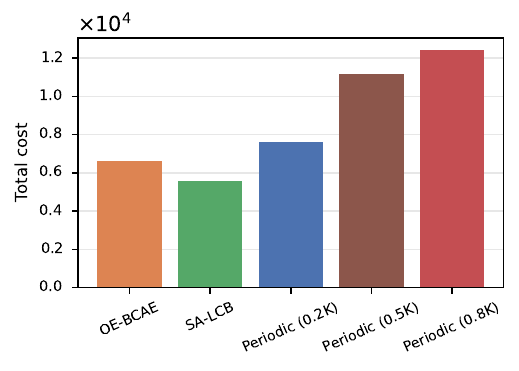}
    \caption{Total-cost comparison of \textsf{OE-BCAE}, \textsf{SA-LCB}, and periodic baselines on the real frequency trace.}
    \label{fig:frequency_compare_periodic}
    \vspace{-2mm}
\end{figure}

\section{Empirical Evaluation}\label{sec:empirical}
In this section, we evaluate the performance of our proposed online learning algorithms through an empirical study on a remote monitoring problem using real-world traces.

\subsection{Simulation Setup}
We consider a remote monitoring problem in which a monitor maintains a cached view of an evolving physical process. The remote source decides when to refresh that view under a constant update cost. As the underlying process, we use the public historical frequency dataset released by Fingrid for the Nordic synchronous power system,~\footnote{\url{https://data.fingrid.fi/en/datasets/339}.} which records frequency measurements from substations at multiple locations in Finland. We regularize the measurements onto a 1-second grid. Letting $y_t$ denote the regularized frequency value, we define the degradation cost of using a sample of age $\tau$ at time $t+\tau$ as $|y_{t+\tau}-y_t|$, which captures the monitoring error induced by outdated information. We then compare with the algorithms listed in Sec.~\ref{sec:num_setup} and heuristic periodic algorithms on chronological traces under a constant-cost update model. Our main results use horizon $T=10^6$, maximum update interval $K=100$, and update cost $C=0.01$, and we further report sensitivity to $K \in \{100,300,500,700,900\}$ and $C \in \{0.005,0.01,0.015,0.02,0.025\}$. We choose the update cost to be comparable in magnitude to the degradation cost, so that the objective reflects a genuine trade-off between communication overhead and information staleness across the range of update intervals. Regret is measured relative to the empirically optimal static update interval computed on the realized trace. We demonstrate the results in Fig.~\ref{fig:empirical} and Fig.~\ref{fig:frequency_compare_periodic}.

\subsection{Results and Discussion}


The regret results in Fig.~\ref{fig:empirical} on the real frequency trace are broadly consistent with the trends observed in the numerical evaluation in Sec.~\ref{sec:simulation}. First, our structure-aware methods \textsf{BCAE} and \textsf{OE-BCAE}, as well as \textsf{SA-LCB}, substantially outperform the unstructured or only partially structured baselines \textsf{LCB} and \textsf{ELMI}. This provides further empirical evidence that exploiting problem structure is critical for achieving low regret. Second, when compared with \textsf{SA-LCB}, we find that \textsf{OE-BCAE} becomes increasingly competitive as the search space grows, and eventually overtakes \textsf{SA-LCB} in the larger-$K$ regime considered here. Third, \textsf{OE-BCAE} consistently improves upon \textsf{BCAE}, indicating that the addition of LCB-style optimism helps the elimination-based strategy adapt more effectively in practice.

Fig.~\ref{fig:frequency_compare_periodic} compares the total cost of our adaptive methods with heuristic periodic baselines under the default setting $K=100$, $C=0.01$, and $T=10^6$. We set the update interval for the periodic baselines as $0.2, 0.5, 0.8$ of $K$ respectively. Both adaptive policies, \textsf{SA-LCB} and our \textsf{OE-BCAE}, outperform all three tested periodic rules by a clear margin. Specifically, our \textsf{OE-BCAE} achieves at least 13\% reduction in total cost compared to heuristic periodic baselines.   This indicates that when there is no prior knowledge about the degradation cost, setting a fixed update schedule heuristically may not be efficient. Instead, adaptive policies that exploit problem structures can efficiently adjust their decisions online based on observed costs and gain an empirical advantage in this setting. 

Finally, we note that our empirical evaluation using real-world traces demonstrates that, while our designs rely on the assumptions that the degradation costs are i.i.d. given age and independent across ages, they remain effective in practical scenarios where these assumptions may not hold perfectly. 


\fi

\section{Conclusion}

In this paper, we study the problem of jointly optimizing degradation costs and update expenses by strategically timing the updates, without prior information on the distribution of the degradation cost.  We have formulated this problem as a novel timing multi-armed bandit problem. We identify the challenges of the problem by providing the lower bound of an ablated timing bandit problem. We then explore the problem-specific structures of timing feedback and develop two efficient online algorithms with sublinear, near-optimal regret for it. Our online algorithms derive an update strategy that achieves the sublinear regret bounds relative to the optimal policy with complete knowledge of the degradation cost distribution.  As for future work, it would be interesting to generalize our approach under less restrictive assumptions. For example, we may incorporate techniques from rising bandits~\cite{Rising2022Metelli} to model scenarios where the degradation costs at a given age decrease in time due to technological advancements. More generally, we may adapt techniques from non-stationary bandits, e.g.,~\cite{besbes2014stochastic}, to handle environments without the i.i.d. assumption.
Another interesting direction is to apply our algorithm and ideas to broader scenarios, including the age-of-information optimization in communication networks, product age optimization in business, etc.

\begin{acks}
This work is supported in part by a General Research Fund from Research Grants Council, Hong Kong (Project No.~11200223), a Collaborative Research Fund from Research Grants Council, Hong Kong (Project No.~C1049-24G), and a Start-up Research Grant from The Chinese University of Hong Kong, Shenzhen (Project No.~UDF01004086). The authors would also like to thank the anonymous reviewers for their helpful comments.
\end{acks}

\bibliographystyle{ACM-Reference-Format}
\bibliography{ref}

\ifx \ISTR \undefined
\else
\appendix

\section{Proof of Proposition~\ref{thm:offline}}\label{App:proof_of_offline}

\begin{proof}
We first show that the problem can be framed as a Markov Decision Process (MDP). In such an MDP, it suffices to consider a stationary, deterministic policy that only depends on the current age. Second, we show that such a policy reduces to updating periodically and derive the optimal period. 

First, we model the problem as a Markov Decision Process (MDP) defined by the tuple \((\mathcal{S}, \mathcal{U}, P, c)\), 1) State Space, \(\mathcal{S} = \mathcal{A} \times \mathcal{G}\), where \(\mathcal{A} = \{1, 2, \dots, K\}\) is the finite set of possible ages, and \(\mathcal{G} \subseteq \mathbb{R}_{\ge 0}\) is the Borel space of possible degradation costs. A state at time \(t\) is \(s_t = (A_t, G_t)\); 2) Action Space: \(\mathcal{U} = \{0, 1\}\), representing "keep" (0) and "update" (1); 3)
Cost Function, the immediate cost is \(c((a,g), u) = g + C(a)u\), where \(C(a) \ge 0\) is the cost of updating at age \(a\); 4) Transition Kernel, Let \(P(\cdot \mid (a,g), u)\) denote the probability measure of the next state \((A_{t+1}, G_{t+1})\). The age transitions deterministically: \(A_{t+1} = 1\) if \(u=1\), and \(A_{t+1} = \min(a+1, K)\) if \(u=0\). The next degradation cost \(G_{t+1}\) is drawn from a probability measure \(F_{A_{t+1}}(\cdot)\) which depends only on the new age \(A_{t+1}\). Note that we include the degradation cost into the state to make the cost function deterministic over the state and the action, which matches the standard MDP in the literature. Later, we will see that, because the degradation cost depends only on the current age and is independent of future degradation costs, it is sufficient to consider optimal policies that only depend on the age but not the realized degradation cost.   

Follwoing~\cite{hernandez1996average}, we establish the Average Cost Optimality Equation (ACOE). Let \(J^*\) be the optimal long-run average cost, and \(h(a,g)\) be the relative value function~\cite{hernandez1996average}. The ACOE is given by:
\begin{align}
& J^* + h(a, g)\\
\begin{split}
  = & \min_{u \in \{0, 1\}} \bigg\{ c((a,g), u)+ \\
&\quad \quad \mathbb{E}[h(A_{t+1}, G_{t+1}) \mid A_t=a, G_t=g, U_t=u] \bigg\}
\end{split}\\
= &\min_{u \in \{0, 1\}} \left\{ g + C(a)u + \int_{\mathcal{G}} h(a_{u}', g') \, dF_{a_{u}'}(g') \right\},
\end{align}
where \(a_{0}' = \min(a+1, K)\) and \(a_{1}' = 1\). The last equation is obtained by substituting the explicit cost and transition dynamics, noting that $A_{t+1}$ only depends on $a$ and $u$, and $G_{t=1}$ only depends on $A_{t+1}$.

The optimal action is given by:
\[
\pi^*(a, g) \in \arg
\min_{u \in \{0, 1\}} \left\{ g + C(a)u + \int_{\mathcal{G}} h(a_{u}', g') \, dF_{a_{u}'}(g') \right\}
\]
Notice that the current degradation \(g\) appears as an additive constant independent of the action \(u\) and we let \(Q(a, u) = C(a)u + \int_{\mathcal{G}} h(a_{u}', g') \, dF_{a_{u}'}(g')\). We have that the optimal action is given by:
\[
\pi^*(a, g) \in \arg\min_{u \in \{0, 1\}} Q(a, u)
\]
Because \(Q(a, u)\) is strictly a function of \(a\) and \(u\), and is completely independent of \(g\), the minimizing action \(u^*\) depends only on \(a\). Therefore, \(\pi^*(a, g) = \pi^*(a)\). 

We can then focus on stationary, deterministic policy \(\pi^*(a)\) that maps the current age \(a_t \in \{1, 2, \dots, K\}\) to an action \(u_t \in \{0, 1\}\). Under such a policy, the system will always update exactly when the age reaches some \(k \in \{1, 2, \dots, K\}\), which can be modeled as a renewal reward process.  By the Elementary Renewal Reward Theorem~\cite{ross2014introduction}, the long-run expected average cost of a periodic policy with period \(k\), is almost surely equal to the expected cost of one cycle divided by the expected length of one cycle, i.e., $  \frac{\sum_{j=1}^{k} \tilde{g}(j) + C(k)}{k}$. Then, the optimal average cost is the minimum one over all possible choices of \(k\). And the optimal period \(k^*\) is the argument that achieves this minimum. This completes the proof. 

\end{proof}

\section{Comparison with the Adversarial Setting~\cite{tripathi2019whittle,tripathi2021onlinearXiv}}\label{app:compareTAoI}
The authors in~\cite{Tripathi2021AoI} introduce an epoch-based formulation, where a strategy fixes one update interval in an epoch and receives feedback on its cost, and consider the adversarial online setting where the degradation cost or the penalty on age is chosen adversarially. In more detail,  they show that the EXP3 algorithm achieves a regret of $O(\sqrt{T'M'\log M'})$, where $T'$ is the number of epochs and $M'$ is the length of the epochs (and also the maximum update interval). They consider that the total cost in an epoch is normalized to be within $[0,1]$. However, when $M'$ is large, such a normalization implies that the cost at each slot approaches zero, which may not be practical. Mapping to our case, the cost in an epoch will be in the order of $M'$. Taking it into account, their regret becomes $O(M'\sqrt{T'M'\log M'})$. Noting the total number of slots $T=T'\cdot M'$ and $K=M'$, applying their result to our case achieves a regret $\tilde{O}(K\sqrt{T})$. 

This is consistent with our finding in Proposition~\ref{thm:structure_lower_bound} that online algorithms not exploiting the consecutive feedback structure of the timing bandit could not achieve an expected regret smaller than $\Omega(K\sqrt{T}/\log K)$. As a comparison, our online algorithm that explores the feedback structure of the timing bandit paradigm achieves a regret of $\tilde{O}(\sqrt{T})$, as shown in Theorem~\ref{thm:ae_independent}



    
\section{Proof of Proposition~\ref{thm:structure_lower_bound}}\label{app:proof_of_structure_lower_bound}

\begin{proof}
    We now discuss the proof of Proposition~\ref{thm:structure_lower_bound}.  Our proof is adapted from ~\cite{auer2002nonstochastic} by further considering that in our case, pulling an arm $k$ will take $k$ slots and $k$ times of cost to receive one feedback.  
    {We consider the following $K+1$ problem instances and we will show that at least one of them will have a large regret given any online algorithm.
    }
    Let $\mu^{(i)}\in \mathbb{R}^K$ be the vector of the means of the arms in the $i$-th instance, $\forall i\in\{0,1,\cdots K\}$, and $\mu^{(i)}_j=1/2+1/2(1-\mathbbm{1}_{i=j})\Delta, \forall i,j\in [k]$, and $\mu^{(0)}_j=1/2+1/2\Delta,\forall j\in[K]$. 
    {Here, $\Delta$ is a parameter to be determined later.}

    We consider that the cost of arm $j$ under instance $i$ follows a Bernoulli distribution with mean $\mu^{(i)}_j$. {We let $\mathbb{E}_i[\cdot]=\mathbb{E}_{\mu^i}[\cdot]$}. We define $R_i$ as the expected regard of an online algorithm under instance $i$. We define $n_k(T)$ as the number of arm $k$ is pulled. 
    {We note that the optimal arm in instance $i$ is arm $i$, and the optimality gaps of other arms are $\Delta/2$.
    Therefore, the regret of the algorithm can be written as}
    \begin{equation}
        R_i = \sum_{k=1,k\neq i}^{K} k\cdot \mathbb{E}_i[n_k(T)]\Delta,
    \end{equation}

    Also, we have that (which is adapted from Lemma A.1 from~\cite{auer2002nonstochastic}, noting that $n_i[T]$ is a random variable bounded by $T/i$ and the KL divergence between Bernoulli($\mu^{(0)}_i$) and Bernoulli($\mu^{(i)}_i$) is bounded by $2\Delta^2$.)
    \begin{equation}
        \mathbb{E}_i[n_i(T)]]\leq \mathbb{E}_0[n_i(T)]+\frac{T}{i}\sqrt{\mathbb{E}_0[n_i(T)]\Delta^2},
    \end{equation}

Then, by the identity that $\sum_{i=1}^{K}i\cdot\mathbb{E}_{0}[n_i(T)]=T$ and Cauchy–Schwarz inequality, we have 
\begin{align}
   & \quad \sum_{i=1}^{K} i\cdot \mathbb{E}_i[n_i(T)] \\
   & \leq \sum_{i=1}^{K} i\cdot\left(\mathbb{E}_0[n_i(T)]+\frac{T}{i}\sqrt{\mathbb{E}_0[n_i(T)]\Delta^2}\right)\\
    &= T + \sum_{i=1}^{K} T\sqrt{\mathbb{E}_0[n_i(T)]\Delta^2}\\
    &= T+ T\Delta \sum_{i=1}^{K} \sqrt{\frac{1}{i}i\mathbb{E}_0[n_i(T)]}\\
    &\leq T+T\Delta\sqrt{\sum_{i=1}^{K} \left(\sqrt{\frac{1}{i}}\right)^2}\sqrt{\sum_{i=1}^{K}{\left(\sqrt{i\mathbb{E}_0[n_i(T)]}\right)^2}}\\
    & \leq  T+ T\Delta\sqrt{T}\sqrt{\log K}
\end{align}    

Then we have that
\begin{align}
    \sum_{i=1}^{K} R_i & =\Delta\sum_{i=1}^{K}\left(T-i\cdot\mathbb{E}_{i}[n_i(T)]\right) \\
    &\geq \Delta\cdot ( K\cdot T-T-T\Delta\sqrt{T}\sqrt{\log K})\\
    & \geq \frac{K^2}{16} \sqrt{\frac{T}{\log K}}
\end{align}
where the last equality is achieved by taken $\Delta=\frac{1}{4}\frac{K}{\sqrt{T\log K}}$ and noting that $K\geq 2$.  We conclude that exist an $i\in[K]$ such that $R_i\geq \frac{1}{16}K\sqrt{T/\log K}$.
\end{proof}

\section{Proof of Lemma~\ref{thm:more_information}}\label{app:thm:more_information}

\begin{proof}[Proof of Lemma~\ref{thm:more_information}]
    We first define some notations. Not that there are $m_k(t)$ feedback of the cost of arm $k$ at slot $t$.  For each arm $k$, we use $h_i(j)$ to denote the random variable of $g(j)$ (or the distribution $f_j$) at its $i$-th update. And we use  $\mu_{k,i}$ as the $i$-th sample of $\mu_{k}$; that is
    \begin{equation}
        \mu_{k,i}= \frac{1}{ k}\left(C(k)+\sum_{j=1}^{k} h_i(j)\right)
    \end{equation}     
    Then, we have that
    \begin{align}
        \bar{\mu}_k(t)&=\frac{1}{m_k(t)}\sum_{i=1}^{m_k(t)}{\mu_{k,i}} \\
        & = \frac{1}{k\cdot m_k(t)} \sum_{i=1}^{m_k(t)}\sum_{j=1}^{k} h_i(j) + \frac{ C(k)}{ k}.
    \end{align}
Noting that
\begin{align}
   \mu_k & = \frac{1}{k\cdot m_k(t)} \sum_{i=1}^{m_k(t)}\sum_{j=1}^{k} \tilde{g}(j)+ \frac{ C(k)}{ k}\\
   &= \frac{1}{k\cdot m_k(t)} \sum_{i=1}^{m_k(t)}\sum_{j=1}^{k} \mathbb{E}[h_i(j)]+ \frac{ C(k)}{ k}\\
   & = \mathbb{E}\left[ \frac{1}{k\cdot m_k(t)} \sum_{i=1}^{m_k(t)}\sum_{j=1}^{k} h_i(j) + \frac{ C(k)}{ k}\right].
\end{align}
We have that 
\begin{align}
   & \; \mathbf{P}\left(|\bar{\mu}_k(t)-\mu_k|\geq r_k(t)\right)\\
   = & \;  \mathbf{P}\Bigg(\frac{1}{k\cdot m_k(t)}\left|\sum_{i=1}^{m_k(t)}\sum_{j=1}^{k}h_i(j)-\sum_{i=1}^{m_k(t)}\sum_{j=1}^{k}\mathbb{E}\left[h_i(j)\right]\right|\\
  & \quad \quad \geq r_k(t)\Bigg)\\
   \leq & \; 2\exp(-2\cdot k\cdot m_k(t)\cdot r^2_k(t))
\end{align}

where the final step follows Hoeffding's inequality by noting that ${h_i(j)},\forall i\in[m_k(t)], j\in[k]$ are independent and bounded random variables in $[0,M]$. And we show the lemma by putting the value of $r_k(t)$ in the last step, which leads to 
\begin{equation}
    \mathbf{P}\left(|\bar{\mu}_k(t)-\mu_k|\geq r_k(t)\right) \leq 2\cdot K^{-1}\cdot T^{-2}
\end{equation}
\end{proof}

\section{Proof of Lemma~\ref{thm:comparable_c_i}}\label{app:thm:comparable_c_i}

\begin{proof}[Proof of Lemma~\ref{thm:comparable_c_i}]
    We prove it by showing that if arm $i$ and $j$ remain in $\mathcal{A}$ at the end of a super-round, we have
    \begin{equation}
        \frac{i\cdot \delta_{m_i}}{j\cdot \delta_{m_j}}\leq 2,
    \end{equation}
    where $\delta_{m_i}$ represents the increment in the number of samples for arm $i$ that are obtained during the super-round. We do not specify the super-round in the notation for ease of presentation.

In Algorithm~\ref{alg:ae_independent}, we have that, for each arm $i$ in $\mathcal{A}$ (which remains in $\mathcal{A}$ after the super-round), we will obtain one sample if we pull an arm with index larger than $i$.  In a super-round, we have that $\tilde{i}$ (the arm we pull) will be no smaller than $i$, when $i\leq 2^{\hat{j}-\tilde{j}-1}$, according to Line~\ref{line:arm_pulling_b}.  And for each $\tilde{j}$, we have $\hat{z}=2^{\tilde{j}}-2^{\tilde{j}-1}$ number of arm pulling. We have that,
\begin{equation}\label{eq:increment_BCAE}
    \delta_{m_i} = 1+ \sum_{\tilde{j}=1}^{\lfloor\hat{j}-\log_2{i}\rfloor}\left(2^{\tilde{j}}-2^{\tilde{j}-1}\right) = 2^{\lfloor\hat{j}-\log_2{i}\rfloor}=2^{\hat{j}-\lceil\log_2{i}\rceil}.
\end{equation}

And, we have that,
\begin{equation}\label{eq:ref_balanced}
    \frac{i\cdot\delta_{m_i}}{j\cdot\delta_{m_j}}=\frac{2^{\hat{j}-\lceil\log_2{i}\rceil+\log_2{i}}}{2^{\hat{j}-\lceil\log_2{j}\rceil+\log_2{j}}}\leq 2.
\end{equation}
And, we conclude the lemma.
\end{proof}

\section{Proof of Lemma~\ref{thm:bounded_sample}}\label{app:thm:bounded_sample}
\begin{proof}[Proof of Lemma~\ref{thm:bounded_sample}]
   For clarity, we denote $\mathcal{A}(t)$ as the set $\mathcal{A}$ at the beginning of slot $t$. Recall that $t_s$ represents the beginning slot of the $s$-th super-round. According to Lemma~\ref{thm:comparable_c_i}, we have that 
    \begin{equation}\label{eq:super-round}
        \frac{i\cdot m_i(t_s)}{j\cdot m_j(t_s)} \geq \frac{1}{2}, \forall i,j\in\mathcal{A}(t_s).
    \end{equation}

 If an arm, say $j$, is eliminated during or before the super-round, $m_j$ will not increase afterward. We can easily see that 
   \begin{equation}\label{eq:super-round-2}
        \frac{i\cdot m_i(t_s)}{j\cdot m_j(t_s)} \geq \frac{1}{2}, \forall i\in\mathcal{A}(t_s),\forall j\in[K].
    \end{equation}



We first show the following inequality for the number of samples of an arm $i\neq k^*$ at slot $t_s$.
\begin{equation}
     m_i(t_s) \leq \frac{36M^2\log (\sqrt{K}T)}{i\cdot \Delta_i^2} + K.
\end{equation}
Suppose not, we consider the first $t_{s'}$ such that,
\begin{equation}
   m_i(t_{s'}) > \frac{36M^2\log (\sqrt{K}T)}{i\cdot \Delta_i^2}.
\end{equation}

We have that

\begin{align}
     \Delta_i > & 6\cdot \sqrt{\frac{M^2\log (\sqrt{K}T)}{i\cdot m_i(t_{s'})}}\\
     \stackrel{(a)}{\geq} &  2\cdot\sqrt{\frac{M^2\log (\sqrt{K}T)}{k^*\cdot m_{k^*}(t_{s'})}}+ 2\cdot \sqrt{\frac{M^2\log (\sqrt{K}T)}{i\cdot m_i(t_{s'})}},\label{eq:gap_sample}
\end{align}
where (a) is by  \eqref{eq:super-round-2} noting that arm $k^*$ is always in $\mathcal{A}$ under the clean event $\mathcal{E}$. 

Then, we can show that
\begin{align}
    & \bar{\mu}_{i}-\sqrt{\frac{M^2\log (\sqrt{K}T)}{i\cdot m_i(t)}}\\
    \geq &\mu_i-2\cdot \sqrt{\frac{M^2\log (\sqrt{K}T)}{i\cdot m_i(t)}} \\
    = &\mu_{k^*}+\Delta_i-2\cdot \sqrt{\frac{M^2\log (\sqrt{K}T)}{i\cdot m_i(t)}}\\
    \geq & \bar{\mu}_{k^*}+\sqrt{\frac{M^2\log (\sqrt{K}T)}{k^*\cdot m_{k^*}(t)}}- 2\cdot\sqrt{\frac{M^2\log (\sqrt{K}T)}{k^*\cdot m_{k^*}(t)}}\\
    & +\Delta_i-2\cdot \sqrt{\frac{M^2\log (\sqrt{K}T)}{i\cdot m_i(t)}}\\ 
  \stackrel{\eqref{eq:gap_sample}}{>} & \bar{\mu}_{k^*}+ \sqrt{\frac{M^2\log (\sqrt{K}T)}{k^*\cdot m_{k^*}(t)}}.
\end{align}
We will have the elimination condition (Line~\ref{line:thresholding_b} in Algorithm~\ref{alg:ae_independent}) holds, and arm $i$ will be eliminated, and $m_i$ will remain unchanged afterward. Noting that we at most obtained $K$ samples of arm $i$ and due to the definition of $s'$ we choose, we have that,   
\begin{equation}
 m_i(t_{s'})\leq m_i(t_{s'-1})+K\leq \frac{36M^2\log (\sqrt{K}T)}{i\cdot \Delta_i^2} + K.
\end{equation}
And we have that $m_i(T)$ is bounded by $m_i(t_s)$ at the complete last super-round, plus the samples obtained from that to $T$ (which are bounded by $K$). We conclude the lemma. 
\end{proof}

\section{Proof of Lemma~\ref{thm:subarms-regret}}\label{app:thm:subarms-regret}

\begin{proof}[Proof of Lemma~\ref{thm:subarms-regret}]

We consider the subset of arms $\mathcal{K}_{j}$.  We denote $K'=\min\{K,2^{j}\}$, which is the maximum index of the arms in $\{\mathcal{K}_{j}$. Note that, according to Line~\ref{line:arm_pulling_b_b} in Algorithm~\ref{alg:ae_independent}, an arm $k$ in $\mathcal{K}_{j}$ will only be pulled when all arms with larger indexes have been eliminated.  If arm $k$ has never been pulled in the total horizon, we have $n_k(T)=0$. Otherwise, suppose the $t'$ is the slot arm $k$ is eliminated; otherwise, we set $t'=T$. We have that the total number of pulling arms $k$ in the whole period, $n_k(T)=n_k(t')$, as arm $k$ will not be pulled afterward. Also, all arms in $\mathcal{K}_{j}$ with index larger than $i$ have been eliminated, and thus, $n_i(T)=n_i(t'),\forall i\in\{i>k|i\in \mathcal{K}_{j}\}$. Also, we will not obtain sample of arm $k$ afterward; thus, $m_i(T)=m_i(t')= \sum_{i=k}^{K} n_i(t')$, as each pulling of arms with indexed larger than $k$ will provide a sample for arm $k$ if arm $k$ haven't been eliminated. We have that 

\begin{equation}
     \sum_{i=k}^{K'} n_i(T) \leq  \sum_{i=k}^{K'} n_i(T) + \sum_{i=K'}^{K} n_i(t') =m_i(T).
\end{equation}
We then apply Lemma~\ref{thm:bounded_sample} to the above inequality.

In summary, we have the following inequality.

\begin{equation}\label{eq:condition_number_of_pulls_app_b}
   n_k(T)=0\; \text{or}\; \sum_{i=k}^{K'} n_i(T) \leq \frac{36M^2\log (\sqrt{K}T) }{k\cdot \Delta_k^2} + 2\cdot K,\forall k\in \mathcal{K}_{j}.
\end{equation}

In addition, it is not hard to see that 
\begin{equation}\label{eq:condition_number_of_pulls_app_b_2}
    \sum_{i\in\mathcal{K}_j} i \cdot n_i(T)\leq T.
\end{equation}

To derive an upper bound to $\sum_{i\in\mathcal{K}_j} i\cdot n_i(T)\cdot \Delta_i$, the left-hand side of~\eqref{eq:subarms_regret}, we consider an optimization problem to maximize it by optimizing $n_i(T),\forall i \in \mathcal{K}_j$ while subject to constraint~\eqref{eq:condition_number_of_pulls_app_b} and constraint~\eqref{eq:condition_number_of_pulls_app_b_2}. 

We separately consider arms with large cost differences where condition~\eqref{eq:condition_number_of_pulls_app_b} tends to dominate and those with small cost differences where condition~\eqref{eq:condition_number_of_pulls_app_b_2} will likely dominate. With such a consideration, we separate arms in $\mathcal{K}_j$ into the following two sets: 
\begin{equation}\label{eq:set_define1_b}
    \mathcal{I}_1 := \left\{ i\in \mathcal{K}_j\left| \frac{1}{\Delta_i} \leq \sqrt{\frac{T}{36M^2\log (\sqrt{K}T)}} \right. \right\},
\end{equation}
and
\begin{equation}\label{eq:set_define2_b}
    \mathcal{I}_2 :=  \left\{ i\in \mathcal{K}_j \left| \Delta_i < \sqrt{\frac{36M^2\log (\sqrt{K}T)}{ T }} \right. \right\}.
\end{equation}

We consider two relaxed constraints to~\eqref{eq:condition_number_of_pulls_app_b} and condition~\eqref{eq:condition_number_of_pulls_app_b_2},
\begin{equation}\label{eq:condition_number_of_pulls_app_relaxed}
    n_k(T)=0\;\text{or} \; \sum_{i\geq k,i\in\mathcal
    {I}_1}^{K'} n_i(T) \leq \frac{36M^2\log (\sqrt{K}T) }{\Delta_k^2} + 2K,\forall k\in\mathcal{I}_1.
\end{equation}
and 
\begin{equation}\label{eq:condition_number_of_pulls_app_2_relaxed}
     \sum_{k\in\mathcal{I}_2} k \cdot n_k(T) \leq T.
\end{equation}

We can see that optimizing $\sum_{i\in\mathcal{K}_j} i\cdot n_i(T)\cdot \Delta_i$ under the relaxed constraints will provide a greater objective value, and we can separate optimizing the $n_{i}(T)$ in $ \mathcal{I}_1$ and $\mathcal{I}_2$ in the relaxed problem.

We first considers optimizing $n_{i}(T)$ in $ \mathcal{I}_1$ subject to~\eqref{eq:condition_number_of_pulls_app_relaxed} to maximizing 
\begin{equation}\label{eq:regret_1_I}
     \hat{R}_1(T):= \sum_{k\in\mathcal{I}_1} k\cdot n_{k}(T)\cdot \Delta_k.
\end{equation}
We denote such an optimization problem as $\mathcal{P}_1$. To find its optimal solution, we identify a special structure for the optimal solution, which is shown in Lemma~\ref{thm:optimal_conditions_b}. We consider one optimal solution to~$\mathcal{P}_1$ that satisfies the property stated in Lemma~\ref{thm:optimal_conditions_b}, i.e.,  $j\cdot \Delta_j$ and $j\cdot\Delta_j^2$ is a strictly increasing sequence in the index of the arms in $\mathcal{G} = \left\{ j\in\mathcal{I}_1|n_j(T) > 0 \right\}$.  We note that fixing $\mathcal{G}$, the problem becomes to maximizing
\begin{equation}\label{eq:regret_1_G}
     \hat{R}_1(T):= \sum_{k\in\mathcal{G}} k\cdot n_{k}(T)\cdot \Delta_k,
\end{equation}
subject to
\begin{equation}\label{eq:condition_number_of_pulls_app_relaxed_G}
    \sum_{i\geq k,i\in\mathcal
    G} n_i(T) \leq \frac{36M^2\log (\sqrt{K}T) }{k\cdot \Delta_k^2} + 2K,\forall k\in\mathcal{G}.
\end{equation}
with decision variables $n_j(T)>0,\forall j\in\mathcal{G}$.
Moreover, according to Lemma~\ref{thm:optimal_conditions_b}, we have that $i\cdot \Delta_{i}<k\cdot \Delta_{k}$ and $i\cdot \Delta^2_{i}<k\cdot \Delta^2_{k}$ if $i<k$. We list the elements in~$\mathcal{G}$ (note that $j\in \mathcal{G}$ if and only if $n_j(T)>0$) in increasing order, i.e., $k_1\leq k_2\cdots\leq k_{|\mathcal{G}|}$. Noting that fixing $\mathcal{G}$, $\mathcal{P}_1$ is an LP with increasing weight in $n_k(T)$ and tightening budget limits in $k$, and the optimal solution is to allocate as most the budget as possible to larger $k$. We can easily show that the optimal solution to the $\mathcal{P}_1$ is by setting, 
\begin{equation}
    n^*_{k_i}(T)= {36M^2\log (\sqrt{K}T)}\cdot\left(\frac{1}{k_i\cdot\Delta^2_{k_i}}-\frac{1}{k_{i+1}\cdot\Delta^2_{k_{i+1}}}\right),\forall i < |\mathcal{G}|,
\end{equation}
and 
\begin{equation}
      n^*_{k_{|\mathcal{G}|}}(T)= \frac{{36M^2\log (\sqrt{K}T)}}{\Delta^2_{k_{{|\mathcal{G}|}}}}+2K.
\end{equation}

We have the following argument,
\begin{align}
 & \quad \hat{R}_1(T) \\
 & \leq \sum_{i=1}^{|\mathcal{G}|} k_i\cdot n^*_{k_i}(T)\cdot \Delta_{k_i}\\
   & {\leq}  \sum_{i=1}^{|\mathcal{G}|-1} k_i\cdot  \Delta_{k_i} \cdot {36M^2\log (\sqrt{K}T)}\cdot\left(\frac{1}{k_i\cdot\Delta^2_{k_i}}-\frac{1}{k_{i+1}\cdot\Delta^2_{k_{i+1}}}\right)\nonumber\\
   & \quad +k_{|\mathcal{G}|}\cdot \Delta_{k_{|\mathcal{G}|}} \cdot \left({36M^2\log (\sqrt{K}T)}\cdot\frac{1}{k_{|\mathcal{G}|}\cdot\Delta^2_{k_{|\mathcal{G}|}}}+K\right)\\
   & = 36M^2\log (\sqrt{K}T) \cdot\Bigg( \frac{1}{\Delta_{k_1}} +\sum_{i=2}^{|\mathcal{G}|} \frac{k_i\cdot \Delta_{k_i}-k_{i-1}\cdot\Delta_{k_{i-1}}}{k_i\cdot\Delta^2_{k_i}}\Bigg)\nonumber\\ 
   &\quad +K\cdot k_{|\mathcal{G}|}\cdot\Delta_{k_{|\mathcal{G}|}}\\
   & \stackrel{\eqref{eq:set_define1_b}}{\leq} 36M^2\log (\sqrt{K}T) \cdot\Bigg(\sqrt{\frac{T}{36M^2\log (\sqrt{K}T)}}\nonumber\\
   &\quad +\sum_{i=2}^{|\mathcal{G}|} \sqrt{\frac{T}{36M^2\log (\sqrt{K}T)}}\frac{1}{k_{i}\Delta_{k_i}}(k_i\Delta_{k_i}-k_{i-1}\Delta_{k_i-1})\Bigg)\nonumber\\
   &\quad+2K\cdot k_{|\mathcal{G}|}\cdot\Delta_{k_{|\mathcal{G}|}}\\
   &\leq \sqrt{36 T M^2\log (\sqrt{K}T)} \cdot\left(1+\log\frac
   {k_{|\mathcal{G}|}\cdot\Delta_{k_{|\mathcal{G}|}}}{k_1\cdot\Delta_{k_1}}\right) \nonumber\\
   &\quad + 2K^2(\bar{C}+M)\\
    &\leq \sqrt{36 T M^2\log (\sqrt{K}T)}\left(1+\log\frac
   {2K(\bar{C}+M)\sqrt{T}}{\sqrt{36M^2\log (\sqrt{K}T)}}\right) \nonumber\\
   &\quad+2K^2(\bar{C}+M)
   \end{align}
where the last equality is by noting that $k_1\geq 1$, $k_{|\mathcal{G}|}\leq K$ and $\Delta_{k_{|\mathcal{G}|}}\leq \bar{C}+M$ (following the definition of $\mu_{k}$).

As for the arms in $\mathcal{I}_2$, we can have that
\begin{align}
{\hat{R}_2(T)} & := \sum_{k\in\mathcal{I}_2} k\cdot n_{k}(T)\cdot \Delta_{k} \\
& \leq \sum_{k\in\mathcal{I}_2} k\cdot n_{k}(T)\cdot \sqrt{\frac{36M^2\log (\sqrt{K}T)}{T}}\\
& \leq  \sqrt{36TM^2\log (\sqrt{K}T)}. 
\end{align}
where the first inequality is by that all arms in $\mathcal{I}_2$, \\$ \Delta_{k} \leq \sqrt{{36M^2\log (\sqrt{K}T)}/{T}}$, and the last inequality is due to $\sum_{k\in\mathcal{I}_2} k\cdot n_{k}(T)\leq T$.

As we discussed earlier in this proof, combining ${\hat{R}_1(T)}$ and ${\hat{R}_2(T)}$ will provide an upper bound to $\sum_{i\in\mathcal{K}_j} i\cdot n_i(T)\cdot \Delta_i$, and we conclude the lemma.
    
\end{proof}

\section{An optimal structure to $\mathcal{P}_1$}

\begin{lemma}\label{thm:optimal_conditions_b}
    There exists an optimal solution $\left\{ n_i(T) \right\}_{i\in \mathcal{K}_1}$ to $\mathcal{P}_{1}$ such that $i\cdot\Delta^2_i$ is a strictly increasing sequence in the index of the arms in $\mathcal{G} = \left\{ i|n_i(T) > 0 \right\}$. Further, $i\cdot\Delta_i$ is also a strictly increasing sequence in the index of the arms in $\mathcal{G}$.
\end{lemma}

\begin{proof}[Proof of Lemma~\ref{thm:optimal_conditions_b}]

We prove it by showing that eliminating the diminishing sequence in 
$\left\{  \Delta_{j}\right\}_{j \in \mathcal{G}}$ will not decrease the objective value. We consider an optimal solution $\left\{ n_j(T) \right\}_{j\in \mathcal{K}_1}$. Suppose that $j\cdot\Delta_{j}^2$ is not a non-decreasing sequence in the index of the arms $\{j|j\in\mathcal{G}\}$. Then we can find two adjacent arms in $\mathcal{G}$, denoted as $j_1<j_2$, such that $j_1\Delta^2_{j_1}\geq j_2 \Delta^2_{j_2}$, which is called an \emph{inversion pair}. Also, according to condition~\eqref{eq:condition_number_of_pulls_app_b}, we have 

\begin{equation}\label{eq:condition_j_1_b}
   \sum_{i\geq j_1,i\in\mathcal{G}}^{K} n_i(T) \leq \frac{36M^2\log (\sqrt{K}T) }{j_1\cdot \Delta_{j_1}^2} + 2K,
\end{equation}
and
\begin{equation}\label{eq:condition_j_2_b}
   \sum_{i\geq j_2,i\in\mathcal{G}}^{K} n_i(T) \leq \frac{36M^2\log (\sqrt{K}T) }{j_2\cdot \Delta_{j_2}^2} + 2K.
\end{equation}
Clearly,~\eqref{eq:condition_j_1_b} implies~\eqref{eq:condition_j_2_b}.
Next, we will show that under such a case, we have
\begin{equation}
j_1\cdot \Delta_{j_1}\leq j_2\cdot \Delta_{j_2}, \label{eq:inversion_pair1_b}
\end{equation}
and we can, therefore, eliminate this inversion pair by constructing a new optimal solution.

We prove it by contradiction. Suppose that $j_1\cdot \Delta_{j_1}>j_2\cdot \Delta_{j_2}$. We construct an alternative solution where $\tilde{n}_{j_1}(T)={n}_{j_1}(T)+{n}_{j_2}(T)$, $\tilde{n}_{j_2}(T)=0$ and we keep all other $n_j(T)$. It is not hard to check that the constructed solution is also feasible to $\mathcal{P}_1$ with a strictly larger objective than the previous one. It contradicts that the original solution is optimal to $\mathcal{P}_1$. We therefore have~\eqref{eq:inversion_pair1_b}.

We consider an alternative solution that $\tilde{n}_{j_1}(T)=0$, $\tilde{n}_{j_2}(T)={n}_{j_1}(T)+{n}_{j_2}(T)$ and we keep all other $n_j(T)$. We can easily check that the constructed solution is also feasible to $\mathcal{P}_1$ with an objective no smaller than the previous one. Thus, the newly constructed solution is also optimal. Moreover, $j_1$ is no longer in the index set with non-zero $n_j(T)$. We eliminate at least one inversion pair.

By continuously applying the same way to the optimal solution until there does not exist $j_1<j_2$, such that $j_1\Delta^2_{j_1}\geq j_2\Delta^2_{j_2}$,  we obtain one optimal solution that achieves the desired property. 

Further, we show that $i\cdot\Delta_i$ is also a strictly increasing sequence in the index of the arms in $\mathcal{G}$. We prove it by contradiction. Suppose that $j_1\cdot \Delta_{j_1}>j_2\cdot \Delta_{j_2}$. We construct an alternative solution where $\tilde{n}_{j_1}(T)={n}_{j_1}(T)+{n}_{j_2}(T)$, $\tilde{n}_{j_2}(T)=0$ and we keep all other $n_j(T)$. It is not hard to check that the constructed solution is also feasible to $\mathcal{P}_1$ with a strictly larger objective than the previous one. It contradicts that the original solution is optimal to $\mathcal{P}_1$. 

\end{proof}

\section{Proof of Theorem~\ref{thm:ae_independent_oe}}\label{app:thm:ae_independent_oe}

\begin{proof}[Proof of Theorem~\ref{thm:ae_independent_oe}]
    We denote the  $n_i(t)$ as the number of times that arm $i$ is selected as $\tilde{i}$ at Step~\ref{line:arm_pulling_b_b} of Algorithm~\ref{alg:ae_independent} up to time $t$. We denote $n_{i,k}(t)$ as the number of times that arm $k$ is actually pulled following the LCB algorithm in the forward window of $i$, i.e., \eqref{eq:LCB_in}. Clearly, we have that
    \begin{equation}\label{eq:lcb_length}
        \sum_{k\in \mathcal{W}_i} n_{i,k}(t) = n_i(t), 
    \end{equation}
    and 
    \begin{equation}
       \sum_{i=1}^{K} \sum_{k\in \mathcal{W}_i} n_{i,k}(T)\cdot k = T.
    \end{equation}

We decompose the regret of the algorithm into two parts,
\begin{align}
 R_{OC}(T) & = \sum_{i=1}^{K} \sum_{k\in \mathcal{W}_i} n_{i,k}(T)\cdot k\cdot \Delta_k\\
 & = \underbrace{\sum_{i=1}^{K} \sum_{k\in \mathcal{W}_i} n_{i,k}(T)\cdot k\cdot \Delta_k - 4\cdot \sum_{i=1}^{K}  n_{i}(T)\cdot i\cdot \Delta_i}_{\text{Part I}: R_I} \nonumber\\ 
 & + \underbrace{4\cdot \sum_{i=1}^{K}  n_{i}(T)\cdot i\cdot \Delta_i}_{\text{Part II}: R_{II}}.\label{eq:regret_decoposed}
\end{align}

We first handle part I. For part I, we note that, following LCB, if $\Delta_k>\Delta_i$.
\begin{equation}\label{eq:LCB_num_bound}
    n_{i,k}(T)\leq \frac{4M^2\log (\sqrt{K}T)}{k\cdot(\Delta_k-\Delta_i)^2}. 
\end{equation}
We can show it by contradiction. Suppose no, we would have that at some $t$
\begin{equation}\label{eq:lcb_num}
    m_{k}(t)\geq n_{i,k}(T) > \frac{4M^2\log (\sqrt{K}T)}{k\cdot(\Delta_k-\Delta_i)^2},
\end{equation}
and we have that 
\begin{align}
   & \; \bar{\mu}_{i}-\sqrt{\frac{M^2\log (\sqrt{K}T)}{i\cdot m_i(t)}}  \\
   \leq & \; {\mu}_{i}\\
   = &\; {\mu}_{k} - \Delta_i + \Delta_k\\
   \leq &\; \bar{\mu}_{k} + \sqrt{\frac{M^2\log (\sqrt{K}T)}{k\cdot m_k(t)}}- \Delta_i + \Delta_k\\
   < & \;\bar{\mu}_{k} - \sqrt{\frac{M^2\log (\sqrt{K}T)}{k\cdot m_k(t)}}
\end{align}
where the last equality follows~\eqref{eq:lcb_num}, and thus arm $k$ will not be chosen following LCB defined in~\eqref{eq:LCB_in}.

For ease of presentation, we denote
\begin{equation}
    R_{I,i}\triangleq \sum_{k\in \mathcal{W}_i} n_{i,k}(T)\cdot k\cdot \Delta_k - 4\cdot  n_{i}(T)\cdot i\cdot \Delta_i
\end{equation}We have that 
\begin{align}
    R_I = \sum_{i=1}^{K} R_{I,i}.
\end{align}

We have that 
\begin{align}
    R_{I,i}\stackrel{(a)}{\leq} & 4\cdot i\cdot \left(\sum_{k\in \mathcal{W}_i} n_{i,k}(T)\cdot \Delta_k - n_{i}(T)\cdot \Delta_i\right)\\
    \stackrel{\eqref{eq:lcb_length}}{=} & 4\cdot i\cdot \sum_{k\in \mathcal{W}_i} n_{i,k}(T)\cdot \left(\Delta_k - \Delta_i\right)\\
    \leq & 4\cdot i\cdot \sum_{k\in \mathcal{W}_i,\Delta_k> \Delta_i} n_{i,k}(T)\cdot \left(\Delta_k - \Delta_i\right)
\end{align}
where (a) is by noting that the arm indexes in the forward window $\mathcal{W}_i$ are bounded by $4\cdot i$.

We separately consider arms in $\mathcal{W}_i$ where $\Delta_k> \Delta_i$, depending on whether $\Delta_k -  \Delta_i\geq \sqrt{\frac{ w\cdot 4M^2\log (\sqrt{K}T) }{ i\cdot n_i(T)}}$. For the "smaller" ones, we have that 
\begin{align}
  &\;  \sum_{k:\Delta_k-\Delta_i\leq \sqrt{\frac{ w\cdot 4M^2\log (\sqrt{K}T) }{i\cdot n_i(T)}}} n_{i,k}(T)\cdot \left(\Delta_k - \Delta_i\right)\\
\leq &\; \sum_{k: \Delta_k-\Delta_i\leq \sqrt{ \frac{w\dot 4M^2\log (\sqrt{K}T) }{i\cdot n_i(T)}}} n_{i,k}(T) \cdot \sqrt{ \frac{w\cdot 4M^2\log (\sqrt{K}T) }{i\cdot n_i(T)}}\\
\leq &\; n_i(T)\cdot \sqrt{\frac{ w\cdot 4M^2\log (\sqrt{K}T) }{i\cdot n_i(T)}}\\
= & \; \sqrt{{n_i(T)}\cdot{w\cdot  4M^2\log (\sqrt{K}T) }/i}.
\end{align}
For the "larger" ones, we have that 
     
\begin{align}
  &\;  \sum_{k:\Delta_k-\Delta_i\geq \sqrt{\frac{w\cdot 4M^2\log (\sqrt{K}T) }{i\cdot n_i(T)}}} n_{i,k}(T)\cdot \left(\Delta_k - \Delta_i\right)\\
\stackrel{\eqref{eq:LCB_num_bound}}{\leq} &\; \sum_{k: \Delta_k-\Delta_i\geq \sqrt{\frac{w\cdot 4M^2\log (\sqrt{K}T) }{i\cdot n_i(T)}}} \frac{w\cdot 4M^2\log (\sqrt{K}T)}{k\cdot(\Delta_k-\Delta_i)} \\
\leq &\; w\cdot  \sqrt{ n_i(T)\cdot\frac{4M^2\log (\sqrt{K}T) }{w\cdot i}}\\
= & \; \sqrt{{n_i(T)}\cdot w \cdot{ 4M^2\log (\sqrt{K}T) }/i}.
\end{align}

Then, for part I, we have that 
\begin{align}
     R_I & = \sum_{i=1}^{K} R_{I,i}\\
     & \leq \sum_{i=1}^{K} \sqrt{{n_i(T)}\cdot w \cdot{ 4M^2\log (\sqrt{K}T) }/i}\\
     & \stackrel{(a)}{\leq} \sqrt{T\cdot w \cdot{ 4M^2\log (\sqrt{K}T) }}\cdot \sqrt{\sum_{i=1}^{K} 1/i^2}\\
     & \leq 4\cdot \sqrt{T\cdot w \cdot{ M^2\log (\sqrt{K}T) }}\label{eq:regret_I}
\end{align}
, where (a) is by noting that $\sum_{i} n_i(T)\cdot i\leq T$ and applying the Cauchy–Schwarz inequality,  and the last equation is by noting that $\sum_{i=1}^{K} 1/i^2\leq 2$. 


We then move on to part II.
\begin{equation}
    R_{II} = 4\cdot \sum_{i=1}^{K}  n_{i}(T)\cdot i\cdot \Delta_i.
\end{equation}
Recall that $n_i(t)$ is the number of times that arm $i$ is selected as $\tilde{i}$ at Step~\ref{line:arm_pulling_b_b} of Algorithm~\ref{alg:ae_independent} up to time $t$. We actually pull an arm following the LCB algorithm in the forward window of $i$, i.e., \eqref{eq:LCB_in}. 

We note that $R_{II}$ is similar to the regret that we have analyzed for the \textsf{BCAE} algorithm. But there is a major difference: at each time to select $\tilde{i}$ at Step~\ref{line:arm_pulling_b_b}, we are based on the past pulling decisions that further involve the LCB component~\eqref{eq:LCB_in}, instead of those selected at previous rounds of Step~\ref{line:arm_pulling_b_b}. Nevertheless, we can show that key properties for demonstrating the regret for \textsf{BCAE} still hold in the analysis of $R_{II}$ (with only differences in some constant factors).

The first property is the balanced confidence bound among all arms, i.e., analogous to Lemma~\ref{thm:comparable_c_i}. Specifically, we can show that 
\begin{equation}\label{eq:super-round}
        \frac{i\cdot m_i(t_s)}{j\cdot m_j(t_s)} \leq 4, \forall i,j\in\mathcal{A}(t_s).
    \end{equation}
The constant increases for two to four, compared to Lemma~\ref{thm:comparable_c_i}. Nevertheless, this will only incur a larger constant factor in all our later analyses. As we showed in Appendix~\ref{app:thm:comparable_c_i}, the critical step is to show that at the end of a super-round, we have
    \begin{equation}\label{eq:per-round}
        \frac{i\cdot \delta_{m_i}}{j\cdot \delta_{m_j}}\leq 4,
    \end{equation}
where $\delta_{m_i}$ represents the increment in the number of samples for arm $i$ that are obtained during the super-round. As the arm we actually pull is not smaller (in index) than that in the \textsf{BCAE}, we have that $\delta_{m_i}$ is at least that in the \textsf{BCAE} (in ~\eqref{eq:increment_BCAE}), i.e., 
\begin{equation}\label{eq:tobe}
    \delta_{m_i}\geq 2^{\hat{j}-\lceil\log_2{i}\rceil}.
\end{equation}
In addition, we get more samples of arm $i$ if we actually pull the larger arm. This only happens when $\tilde{i}$ at Step~\ref{line:arm_pulling_b_b} of Algorithm~\ref{alg:ae_independent} is within $[2^{\lceil\log_2 i\rceil-2}+1,2^{\lceil\log_2 i\rceil-1}]$ following the definition of the forward window in~\eqref{eq:defn:window} and $2^{\hat{j}-\tilde{j}}=2^{\lceil\log_2 i\rceil-1}]$.\footnote{According to Step~\ref{line:arm_pulling_b_b} of Algorithm~\ref{alg:ae_independent}, if $\tilde{j}$ is larger, such an $\tilde{i}$ is too larger to be chosen. If $\tilde{j}$ is smaller, $\tilde{i}$ is chosen only when arm $i$ has been eliminated.} The number of such pulling is at most $2^{\tilde{j}}-2^{\tilde{j}-1}$ when $\tilde{j}=\hat{j}-{\lceil\log_2 i\rceil}+1$. We have that 
\begin{equation}
    \delta_{m_i}\leq 2^{\hat{j}-\lceil\log_2{i}\rceil}+2^{\hat{j}-\lceil\log_2{i}\rceil}.
\end{equation}
With the lower bound and upper bound above, we can easily conclude~\eqref{eq:per-round} (similar to ~\eqref{eq:ref_balanced} and~\eqref{eq:super-round}. Further, with~\eqref{eq:super-round}, we can prove the version of Lemma~\ref{thm:bounded_sample} for the analysis of $R_{II}$ (with a larger constant factor).

The second property is the specific bound on $n_i(T)$, analogous to~\eqref{eq:condition_number_of_pulls_app_b_m}, due to the feedback structure. We note that  $n_i(t)$ is the number of times that arm $i$ is selected as $\tilde{i}$ at Step~\ref{line:arm_pulling_b_b} of Algorithm~\ref{alg:ae_independent} up to time $t$. \footnote{We note that while in \textsf{OE-BCAE}, arm $i$ may also be pulled when chosen by the introduced LCB, these are not counted in $n_i(t)$. } Thus, $n_i(t)$ is similar to that under \textsf{BCAE}, we will never choose $i$ at Step~\ref{line:arm_pulling_b_b} unless all the arms in the corresponding subset with larger indexes are eliminated, which results in~\eqref{eq:condition_number_of_pulls_app_b_m} (with a larger constant factor under the analysis of $R_{II}$).

With the above two key properties, we can follow the proof of Lemma~\ref{thm:subarms-regret} and Theorem~\ref{thm:ae_independent} to show that  $R_{II}$ is in $\tilde{O}(\sqrt{T})$. We omit the details for simplicity. Together with~\eqref{eq:regret_I} and~\eqref{eq:regret_decoposed}, we conclude the theorem.


\end{proof} 
\fi

\end{document}